\documentclass{article}

\usepackage[final]{neurips_2026}

\usepackage[utf8]{inputenc}
\usepackage[T1]{fontenc}
\usepackage{hyperref}
\usepackage{url}
\usepackage{booktabs}
\usepackage{amsfonts}
\usepackage{amsmath}
\usepackage{amssymb}
\usepackage{amsthm}
\usepackage{nicefrac}
\usepackage{microtype}
\usepackage{xcolor}
\usepackage{graphicx}
\usepackage{multirow}
\usepackage{adjustbox}
\usepackage{subcaption}
\usepackage{bm}
\usepackage{float}

\makeatletter
\renewcommand{\@notice}{}
\makeatother

\graphicspath{{figures/}}

\newtheorem{theorem}{Theorem}[section]

\newtheorem{proposition}[theorem]{Proposition}

\newcommand{\RR}{\mathbb{R}}
\newcommand{\EE}{\mathbb{E}}

\newcommand{\calV}{\mathcal{V}}
\newcommand{\calE}{\mathcal{E}}
\newcommand{\calN}{\mathcal{N}}

\newcommand{\bfX}{\mathbf{X}}

\usepackage[section]{placeins}

\title{Spectral Reversal: Counteracting Singular Value Bias for Graph Prompting}

\author{%
  Hanxu Yang\thanks{Equal contribution.}, \hspace{0.2cm}
  Yuhuan Zhao\footnotemark[1], \hspace{0.2cm}
  Xiaodong He, \hspace{0.2cm}
  Zhao Kang\thanks{Corresponding author.} \\
  School of Computer Science and Engineering,\\
  University of Electronic Science and Technology of China, Chengdu, Sichuan, China\\
  \texttt{yang\_hanxu@std.uestc.edu.cn \quad zkang@uestc.edu.cn} \\
}

\begin{document}

\maketitle

\begin{abstract}
Pre-training Graph Neural Networks (GNNs) via self-supervised learning has become a dominant paradigm, yet efficiently adapting frozen encoders remains a challenge. Graph prompting offers a parameter-efficient alternative to fine-tuning, but existing methods largely treat pre-trained models as opaque feature extractors, ignoring their internal spectral structure. In this work, we identify a systematic phenomenon in pre-trained GNNs, which we term spectral bias: optimization during pre-training disproportionately aligns representations with directions associated with large singular values, leaving low-energy directions under-explored. We show that these underutilized directions can encode complementary information that is beneficial for downstream adaptation, especially under distribution shift. To leverage this insight, we propose Spectral Reverse Prompt (SRP), a prompting framework that rebalances the spectral contributions of frozen GNN encoders. SRP applies a learnable soft-thresholding mask in the spectral domain to down-weight dominant directions while amplifying weaker ones. In addition, SRP incorporates a null-space augmentation module that captures variation in directions with minimal activation under the frozen encoder. Extensive experiments across multiple benchmarks demonstrate that SRP achieves state-of-the-art performance with minimal additional parameters, highlighting that reweighting spectral components is a principled and effective strategy for parameter-efficient graph adaptation. The code is available at \url{https://github.com/keris-yang/Spectral-Reversal}.
\end{abstract}
\vspace{-1.0em}

\section{Introduction}
\label{sec:intro}

Graph prompt tuning~\cite{sun2023survey,gplsurvey} has emerged as a parameter-efficient alternative to fine-tuning for adapting pre-trained Graph Neural Networks~\cite{kipf2016gcn,velivckovic2017gat,xu2018gin,hamilton2017graphsage,brody2021gatv2} (GNNs). Existing approaches typically introduce learnable tokens~\cite{fang2023gpf,edgeprompt,li2025instanceaware} or modify readout layers~\cite{sun2022gppt,yu2024multigprompt,yu2024hgprompt}, while keeping the backbone encoder frozen. However, these methods implicitly treat the pre-trained model as a black-box feature extractor, overlooking its internal structure.

In particular, little attention has been paid to the spectral properties of the learned transformations within pre-trained GNNs. During self-supervised pre-training, optimization tends to align representations with dominant modes of variation, which are often associated with large singular values of the effective transformation. We refer to this phenomenon as spectral bias, where dominant spectral components are emphasized while lower-energy directions remain underutilized. While these dominant directions capture prevalent patterns in pre-training data, they may not align with the requirements of downstream tasks~\cite{huang2024measuring}, especially under inter-class inequality.

This observation suggests a different perspective on parameter-efficient adaptation: rather than further reinforcing dominant directions, it may be more effective to reweight the spectral contributions of the frozen encoder, allocating capacity to directions that were under-explored during pre-training. To this end, we propose \textbf{Spectral Reverse Prompt (SRP)}, a framework that explicitly rebalances spectral contributions through two complementary mechanisms. First, a spectral channel applies a learnable soft-thresholding mask to modulate singular directions, suppressing dominant components while amplifying weaker ones. Second, a null-space channel extracts informative variation from directions that are minimally expressed by the frozen encoder, further expanding the effective representation space.

By focusing adaptation on underutilized spectral components, SRP improves parameter efficiency while enhancing robustness to distribution shift. Empirical results across diverse benchmarks demonstrate that this spectral rebalancing leads to consistent performance gains over existing graph prompting methods across node-level and graph-level benchmarks. Our main contributions can be summarized as follows:

\begin{itemize}
    \item We identify and quantify a form of spectral bias in pre-trained graph models using a \textbf{Downstream–Pretraining Misalignment Score} (DPMS), showing that dominant singular directions are disproportionately aligned with pre-training objectives, while lower-energy directions remain underutilized for downstream tasks.
.
    \item We propose SRP, a parameter-efficient graph prompting framework that rebalances spectral contributions of frozen encoders via two complementary components: (i) a learnable soft-thresholding spectral mask that suppresses dominant directions and amplifies weaker ones, and (ii) a null-space augmentation module that captures variation in directions minimally expressed by the encoder.
.
    \item We show that emphasizing underutilized spectral components leads to more effective and robust adaptation, particularly in low-resource settings.

\end{itemize}

\section{Related Work}
\label{sec:related}

\paragraph{GNNs and pre-training.}
Numerous studies propose training powerful GNNs via self-supervised learning, which is broadly categorized into two paradigms: contrastive and generative methods. Contrastive methods aim to maximize agreement between augmented instances. Representative works include DGI~\cite{velivckovic2019dgi} and InfoGraph~\cite{sun2019infograph}, which focus on mutual information across different structural scales; GraphCL~\cite{you2020graphcl}, GRACE~\cite{GRACE}, and SimGRACE~\cite{xia2022simgrace} generate contrastive views via structural augmentation, feature masking, or parameter perturbation. Other paradigms avoid negative sampling through bootstrapping~\cite{BGRL} or exploit domain-specific motifs~\cite{zhang2021mgssl}. Conversely, generative methods reconstruct specific graph properties: GraphMAE and GraphMAE2~\cite{hou2022graphmae,GraphMAE2} focus on masked node features, while earlier strategies~\cite{hu2019strategies} leverage edge prediction. This paradigm of pre-training then fine-tuning has become the core of research in graph domain adaptation~\cite{fang2025benefits}, graph foundation models~\citep{wang2024gft,MDGFM,BRIDGE,he2026structurecentric,he2026scgfmartamortizedrelationaltransport} and graph reasoning mechanisms~\cite{xie2026clusteringreasoningkmeansinterpretation} in recent years.

\paragraph{Graph tuning methods.}
Graph prompt tuning and low-rank adaptation, or LoRA, offer parameter-efficient alternatives by optimizing minimal components alongside the backbone. GraphLoRA~\cite{GraphLoRA} extends traditional LoRA~\cite{hu2022lora} by incorporating structure-aware contrastive losses. Meanwhile, graph prompting methods typically inject learnable tokens into various stages while preserving fixed message-passing rules. These approaches target different levels: node-feature prompting via GPF~\cite{fang2023gpf}, edge-level adaptation through EdgePrompt~\cite{edgeprompt}, and topology modifications in GraphTOP\citep{fu2026graphtop}, UniPrompt~\cite{UniPrompt} and CurvPrompt~\cite{wang2026dynamicgraphpromptingtopologyrouted}. Further research explores representation-level prompting such as GraphPrompt~\cite{liu2023graphprompt} and GPPT~\cite{sun2022gppt}, as well as specialized designs: virtual-node prompting~\cite{tan2023vnt} and instance-aware schemes~\cite{li2025instanceaware}. Despite their success, these approaches primarily adapt inputs or structures without explicitly analyzing the consensus between prompting mechanisms and weight metrics. This gap motivates our work: we rebalance spectral contributions and utilize the null space to enable universal parameter efficient, task-driven adaptation.

\section{Notations and Preliminaries}
\label{sec:prelim}

\textbf{General Graphs.} Let $\mathcal{G}=(\mathcal{V},\mathcal{E}, \mathbf{X}, \mathcal{Y})$ be a graph with $N$ nodes, where $\mathcal{V}=\{ v_1, \dots, v_{N} \}$ and $\mathcal{E} \subseteq \mathcal{V} \times \mathcal{V}$ define the node and edge sets, respectively. The node attributes are captured by a feature matrix $\mathbf{X} \in \mathbb{R}^{N \times F}$, with each row $\mathbf{X}_i$ representing the $F$-dimensional feature of node $v_i$. Structural connectivity is denoted by the adjacency matrix $\mathbf{A} \in \{0, 1\}^{N \times N}$, where $\mathbf{A}_{ij}=1$ indicates an existing edge between $v_i$ and $v_j$. Furthermore, each node $v_i$ corresponds to a specific label $y_i$ from the label space $\mathcal{Y}$. Note that $P(\cdot)$ represents a probability distribution, which is utilized conceptually in our context rather than for strict mathematical derivations.

\textbf{Graph Prompt Tuning.} Traditional pretrain-finetuning paradigms jointly update a pretrained graph encoder $f_{\theta}$ and a task-specific head $g_{\phi}$ using downstream data. In contrast, graph prompt learning freezes $f_{\theta}$ and optimizes only a newly introduced set of prompt parameters $\Psi$. The generalized objective for such prompt-based adaptation on a downstream dataset $\mathcal{D}$ is formulated as:
\begin{equation}
    \max_{\Psi} \frac{1}{|\mathcal{D}|} \sum_{(\mathbf{A}, \mathbf{X}, y) \in \mathcal{D}} \sum_{i=1}^{N} \log P\left(y_{i} \mid \text{Predict}_{\Psi}\left(\mathbf{A}, \mathbf{X}, v_{i}; f_{\theta}\right)\right),
\end{equation}
where $\mathcal{D} = \{ ( \mathbf{A}, \mathbf{X}, y )\}$. Here, $\text{Predict}_{\Psi}(\cdot)$ serves as a generic prediction module that leverages the structural topology $\mathbf{A}$, node features $\mathbf{X}$, the target node $v_i$, and the frozen backbone $f_{\theta}$ to estimate the corresponding label $y_i$. Depending on the architecture, the prompt $\Psi$ can be injected at various stages: modifying the raw topological and feature inputs (input-level), embedding into intermediate network layers (layer-wise), or adjusting the final output embeddings before the classifier (representation-level).

\textbf{Singular Value Decomposition.}
For any given weight matrix $\mathbf{W} \in \mathbb{R}^{m \times n}$ with rank $r$, its compact Singular Value Decomposition (SVD) and the induced right null space are given by
\begin{equation}
\mathbf{W}=\mathbf{U}\boldsymbol{\Sigma}\mathbf{V}^{\top},
\qquad
\ker(\mathbf{W})
=\operatorname{span}(\mathbf{V})^{\perp}
=\left\{\mathbf{x}\in\mathbb{R}^{n}\mid \mathbf{V}^{\top}\mathbf{x}=0\right\}.
\end{equation}
In this factorization, $\mathbf{U} \in \mathbb{R}^{m \times r}$ and
$\mathbf{V} \in \mathbb{R}^{n \times r}$ consist of orthonormal columns
known as the left and right singular vectors, respectively. The term
$\boldsymbol{\Sigma} \in \mathbb{R}^{r \times r}$ is a diagonal matrix
containing the sorted, positive singular values
$\{\sigma_i\}_{i=1}^{r}$ of $\mathbf{W}$ in descending order.
Accordingly, the right null space of $\mathbf{W}$ is the orthogonal
complement of the subspace spanned by its right singular vectors.

\section{The Spectral Bias of Pre-trained GNNs}
\label{sec:spectral_bias}

To establish our theoretical framework, we investigate a frozen GNN encoder with a weight matrix $W \in \mathbb{R}^{  d_\text{out} \times d_\text{in}}$ at an arbitrary layer. For downstream adaptation, we assume $W_\text{down} \in \mathbb{R}^{d_\text{in} \times C}$ represents the optimal linear classifier trained directly on raw node features.

\paragraph{Origins of Spectral Bias.}
The spectral bias in pre-trained GNNs fundamentally stems from two
compounded phenomena: the nature of gradient descent and the graph
message-passing mechanism. During pre-training, $W$ is optimized over
$T$ steps of gradient descent. Each step applies a rank-1 update
$\Delta W^{(t)}=-\eta_t g^{(t)}(h^{(t)})^\top$ based on the upstream
gradient $g^{(t)}$ and layer input $h^{(t)}$. For a GNN, let
$h^{(t)}=\phi(\widehat A\widetilde W^{(t)}\widetilde h^{(t)})$ and let
the adjacency matrix be eigendecomposed as
$\widehat A = Q\Lambda Q^\top$, the cumulative update is
\begin{equation}
\resizebox{0.98\columnwidth}{!}{$\displaystyle
\Delta W
= W^{(T)}-W^{(0)}
= -\sum_{t=1}^{T}\eta_t g^{(t)}(h^{(t)})^\top
= -\sum_{t=1}^{T}\eta_t g^{(t)}
\Big[\phi\!\big(Q\Lambda Q^\top \widetilde W^{(t)}\widetilde h^{(t)}\big)\Big]^\top
$}
\label{eq:cumulative}
\end{equation}
Since each update deposits energy along $h^{(t)}$, the spectral
structure of $\Delta W$ is governed by the input covariance
$\Sigma_h=\mathbb{E}_t[h^{(t)}(h^{(t)})^\top]$. Let
$\Sigma_h=\sum_i\lambda_i e_i e_i^\top$
(with $\lambda_1\geq\lambda_2\geq\cdots$) and
$P_r=\sum_{i=1}^{r}e_i e_i^\top$ be the projection onto the top-$r$
principal directions.

\begin{proposition}[Spectral Alignment and Concentration]
If the training inputs satisfy
$\|(I-P_r)h^{(t)}\|_2\leq\delta\|h^{(t)}\|_2$ for all $t$, and the
cross-step gradient terms vanish, then the cumulative gradient update
inherits this concentration:
\begin{equation}
\frac{\|P_r\Delta W^\top\|_F^2}{\|\Delta W\|_F^2}
=1-\frac{\|(I-P_r)\Delta W^\top\|_F^2}{\|\Delta W\|_F^2}
\geq 1-\delta^2 .
\label{eq:gradient_conc}
\end{equation}
\label{prop:spectral_align}
\end{proposition}

The detailed proof is provided in the Appendix~\ref{sec:app_proofs}. Proposition \ref{prop:spectral_align} reveals that when inputs are confined to a low-dimensional subspace, gradient updates restrict their energy to the same principal subspace. Crucially, this gradient-induced bias is compounded by GNN message passing $h^{(\ell)} = \phi(\hat{A}\,W^{(\ell)}h^{(\ell-1)})$, which acts as a graph-spectral filter. Over $L$ layers, scaling frequency components by $\lambda_i^L(\hat{A})$ causes an \emph{anisotropic contraction} of the covariance matrix. This asymptotically forces node representations to collapse toward a dominant vector, severely skewing the spectral distribution of the pre-trained weights.

\textbf{Downstream Capacity and Misalignment.}
For a frozen encoder $W$, the downstream classifier is strictly confined to the \emph{spectral reachable set} $\mathcal{H}(W) := \{h \mapsto a^\top(Wh) : a \in \mathbb{R}^{d_\text{out}}\}$. Since $\text{col}(W^\top) = \text{col}(V)$, the encoder restricts any downstream head to a $k$-dimensional subspace (where $k = \text{rank}(W)$). The minimal approximation error against the optimal downstream classifier $W_\text{down}$ is:
\begin{equation}
  \varepsilon^2(W) := \min_{A} \|W_\text{down} - W^\top A\|_F^2 = \|P_\mathcal{N}\, W_\text{down}\|_F^2,
  \label{eq:approx_bound}
\end{equation}
which is achieved at $A^* = U\Sigma^{-1}V^\top W_\text{down}$, leaving an irreducible residual defined by the null space projection $P_\mathcal{N} = I - VV^\top$. We further find that applying a computationally inexpensive PCA to the null space leads to substantially higher information density than the Encoder. As shown in Figure~\ref{fig:motivation_combined}, using only 8 dimensions, the resulting representation achieves classification performance comparable to that of a 128 dimensional Encoder obtained through costly training.

To explicitly quantify the tension between downstream demand and pre-training capacity, we introduce the \emph{Downstream-Pretraining Misalignment Score} for each singular direction $v_j$. By partitioning the singular values into quintile groups $\{Q_q\}_{q=1}^5$, we further define the aggregate \emph{net gap} ($\Delta_q$):
\begin{equation}
  \mathrm{DPMS}_j := \frac{\|W_\text{down}^\top v_j\|^2}{\sigma_j^2}, \quad 
  \Delta_q := \frac{\sum_{j \in Q_q}\|W_\text{down}^\top v_j\|^2}{\sum_j \|W_\text{down}^\top v_j\|^2} - \frac{\sum_{j\in Q_q}\sigma_j^2}{\sum_j \sigma_j^2}.
  \label{eq:metrics}
\end{equation}

\begin{figure*}[t]
  \centering
  \includegraphics[width=0.48\textwidth]{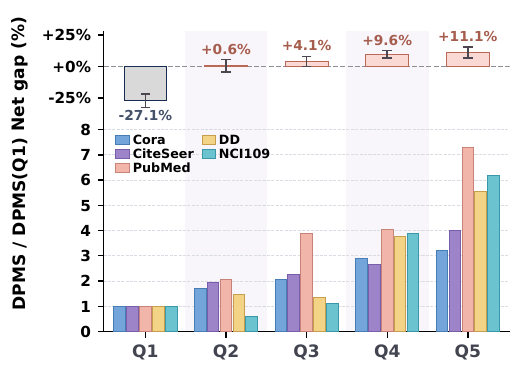}
  \hfill
  \includegraphics[width=0.48\textwidth]{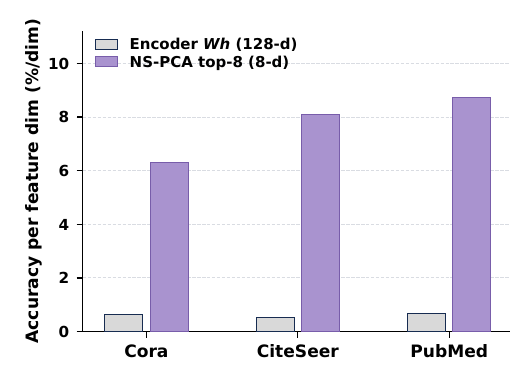}
  \caption{\textbf{Left: DPMS and $\Delta_q$ in the  singular values of the quintices of $W_{down}$.} 
  \textbf{Right: Information density of encoder and null space.} }
  \label{fig:motivation_combined}
\end{figure*}

\begin{theorem}[Spectral Reversal]
\label{thm:reversal}
The interaction between pre-trained capacity and downstream demand exhibits the following properties:
\begin{itemize}
  \item \textbf{Reachability Limits:} Target signals within $\text{col}(V)$ are fully recoverable via $a_j = \sigma_j^{-1}u_j$, whereas signals in the null space $\ker(W)$ yield an irreducible approximation error $\varepsilon^2(W) = \|P_\mathcal{N} W_\text{down}\|_F^2$.
  \item \textbf{Capacity Misalignment:} $\mathrm{DPMS}_j$ grows inversely with the pre-trained capacity $\sigma_j^2$. Consequently, severe spectral concentration, where $\sigma_i \gg \sigma_j$, forces the ratio $\mathrm{DPMS}_j/\mathrm{DPMS}_i$ to diverge, which is regionally quantified by the net gap $\Delta_q$.
\end{itemize}
\end{theorem}

The theorem above formalizes the fundamental tension between the pre-training and fine-tuning phases. The first property indicates that a frozen backbone inherently introduces \emph{blind spots}: any downstream feature reliance that falls into the null space of the pre-trained weights simply cannot be recovered, regardless of how the linear head is optimized. The second property further reveals that even within the reachable subspace, there is a severe mismatch in resource allocation. The metric $\mathrm{DPMS}_j$ acts as a demand-to-supply ratio. Due to spectral bias, the pre-trained model over-allocates its capacity to dominant directions, leading to a low DPMS and a negative net gap representing redundant supply. Conversely, the tail directions, which are often crucial for specific downstream tasks, are starved of capacity, resulting in a critically high DPMS and a positive net gap. This structural misalignment strongly motivates the need for targeted spectral adaptation.

\begin{theorem}[PAC-Bayes Generalization Bound]
\label{thm:pac_main}
Consider a learnable adaptor $\hat{\Theta}$ restricted to a bottleneck rank $r_s$ across exactly $|\mathcal{V}_w|$ weak spectral directions, yielding a parameter count of $d_\Theta = \mathcal{O}(|\mathcal{V}_w| r_s)$. Let $\hat{R}(\hat{\Theta})$ and $R(\hat{\Theta})$ be the empirical and expected risks on $n$ samples. Under McAllester's bound with a Gaussian prior bounded by the backbone's condition number $\kappa(W) := \sigma_\text{max}/\sigma_\text{min}$, and an analytically optimized posterior, with probability $\geq 1-\delta$:
\begin{equation}
  R(\hat{\Theta}) \;\leq\; \hat{R}(\hat{\Theta}) + \mathcal{O}\left( \sqrt{\frac{|\mathcal{V}_w| r_s \log\kappa(W) + \log(n/\delta)}{n}} \right).
  \label{eq:pac_main}
\end{equation}
\end{theorem}

\textbf{Proof Sketch and Motivation.} 
By Theorem~\ref{thm:reversal}, the effective signal scale per target direction is restricted by $1/\sigma_\text{min}$, which strictly bounds the parameter norm as $\|\hat{\Theta}\|^2 = \mathcal{O}(d_\Theta / \sigma_\text{min}^2)$. Optimizing the posterior variance collapses the KL divergence to $\mathcal{O}(d_\Theta \log\kappa(W))$, yielding Equation \eqref{eq:pac_main}.

Crucially, this bound demonstrates that the generalization gap depends explicitly on the number of activated weak directions $|\mathcal{V}_w|$, rather than the massive original feature dimensions $d_\text{in}d_\text{out}$. This offers a profound theoretical foundation: by deliberately restricting trainable capacity to only the highly misaligned subspace, one can rectify the spectral bias while inherently securing a significantly tighter generalization guarantee than full fine-tuning. This structural insight directly motivates our algorithmic design in the subsequent section.

\begin{figure}[t]
  \centering
  \includegraphics[width=\textwidth]{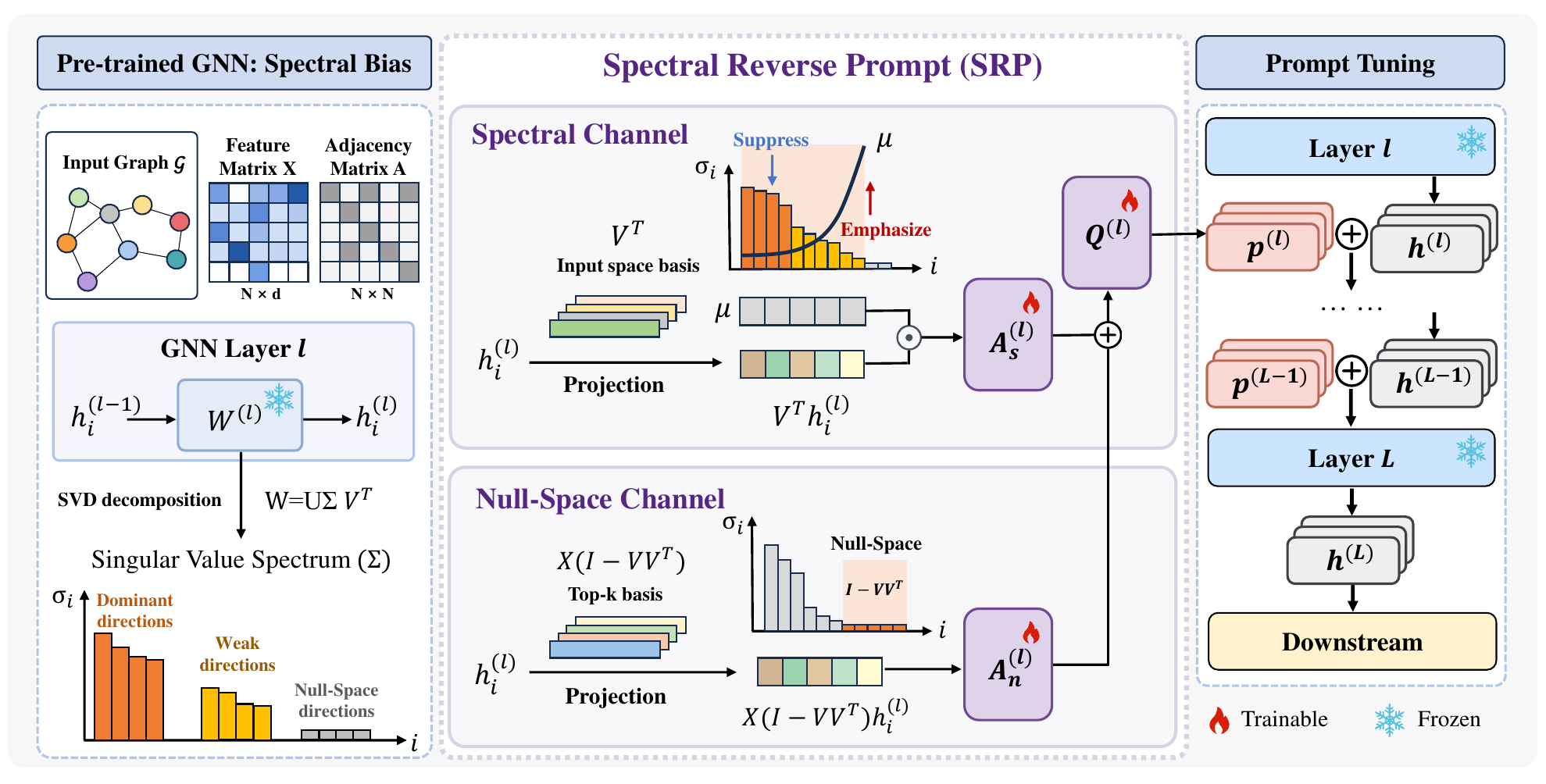}
  \caption{\textbf{Overall framework of Spectral Reverse Prompt (SRP).} 
SRP decomposes each frozen GNN layer weight to expose spectral bias, then uses a spectral channel to suppress dominant directions and amplify weak ones, and a null-space channel to exploit information that is not captured by the frozen encoder.
The two channels are fused into layer-wise prompts injected into the frozen GNN for downstream tuning.
}
  \label{fig:overview}
\end{figure}

\section{Method: Spectral Reverse Prompt}
\label{sec:method}

Motivated by the spectral reversal principle, we design our method, Spectral Reverse Prompt (SRP). Our approach introduces a layer-wise graph prompt that counteracts singular values to better align the pretrained models with downstream few-shot tasks. SRP adds prompting vector to node feature:
\begin{equation}
  \tilde{z}^{(\ell)} = \hat{A}\, W^{(\ell)} h^{(\ell-1)} + b^{(\ell)} + p^{(\ell)} (h^{(\ell-1)}).
\end{equation}
Each prompt consists of two complementary channels on orthogonal subspaces of $\RR^{d_\text{in}}$: a \emph{spectral channel} that reverses the encoder's spectral weighting, and a \emph{null-space channel} that extracts signal from directions invisible to $W$.

\subsection{Spectral Channel: Soft-Threshold Reversal}
\label{subsec:spectral_channel}

Motivated by Theorem~\ref{thm:reversal}, this channel amplifies weak directions while suppressing dominant ones.
Node features are projected onto the right-singular basis: $h_\text{spec} = V^\top h \in \RR^k$ (the spectral coordinates of $h$).
A learnable soft-threshold mask is then applied:
\begin{equation}
  \tau = \mathrm{sigmoid}(\tau_0), \qquad
  \mu_j = \mathrm{sigmoid}\!\big((\tau \sigma_\text{max} - \sigma_j) \cdot \gamma\big),
  \label{eq:mask}
\end{equation}
where $\tau_0$ is a scalar threshold parameter, $\sigma_\text{max}$ the largest singular value of $W$, $\gamma$ is a pre-modulatable temperature coefficient that can be freely adjusted according to specific experiments, and $\mu_j \in (0,1)$ the mask weight for direction $j$.
By construction $\mu_j \approx 0$ for strong directions and $\mu_j \approx 1$ for weak ones.
The masked coordinates are linearly projected: $z_\text{spec} = (h_\text{spec} \odot \mu)\, A_s$, where $A_s \in \RR^{r_s\times k}$ is a learnable bottleneck adaptor of rank $r_s$.

\subsection{Null-Space Channel}
\label{subsec:null_channel}

When $d_\text{in} > k$, the null space $\ker(W)$ has dimension $d_\text{in} - k$.
Features in this subspace are completely invisible to the frozen encoder and inaccessible to the spectral channel.
We project the feature matrix $\bfX$ onto $\ker(W)$ and extract the top-$m$ principal components:
\begin{equation}
  \bfX_{\perp} = \bfX(I - VV^\top), \qquad
  z_\text{null} = A_n\,(V_{\perp}h),
  \label{eq:null_channel}
\end{equation}
where $V_{\perp} \in \RR^{m\times d_\text{in}}$ is the orthonormal PCA basis of the null-space-projected data $\bfX_{\perp}$, and $A_n \in \RR^{r_s \times m}$ is a learnable adaptor.
By construction $WV_{\perp} = 0$ and $V^\top V_{\perp} = 0$: the two channels operate on orthogonal subspaces.
Both $V_{\perp}$ and $V$ are frozen; only $A_n$ is learned.
Figure~\ref{fig:motivation_combined} and proofs given by Appendix~\ref{sec:app_proofs} confirm that NS-PCA directions carry higher per-dimension accuracy than the encoder output and that PCA ordering substantially outperforms random null-space selection.


\subsection{Channel Fusion and Output}
\label{subsec:fusion}

The two channels are summed in a shared bottleneck, then projected to the GNN activation space:
\begin{equation}
  p(h) = Q\,\Big[\underbrace{A_s\,(V^T h \odot \mu)}_{\text{spectral channel}} \;+\; \underbrace{A_n\,(V_{\perp}h)}_{\text{null-space channel}}\Big] + b_p,
  \label{eq:srp}
\end{equation}
where $Q$ and $b_p$ are the shared projection and bias. This fusion is lossless since $V^\top V_{\perp}^\top = 0$. When $d_\text{in} \leq k + m$, SRP uses a projective fallback $p(h) = (h_\text{weak} P) Q + b_p$. Here, $h_\text{weak} = h - (hV \odot (1{-}\mu))V^\top$ directly preserves all weak spectral and null-space signals. Thus, a single adaptor $P \in \mathbb{R}^{d_\text{in} \times r_s}$ becomes mathematically isomorphic to the dual-channel architecture (Appendix~\ref{subsec:appendix_fallback}).

\subsection{Complexity Analysis}
\label{subsec:complexity}
The proposed SRP exhibits exceptional efficiency, characterized by strictly linear time complexity with respect to the graph size and a trainable parameter footprint drastically smaller than standard LoRA-style adaptations.

\textbf{Time Complexity.} Initial SVD and null-space computations are performed once offline. During training, by pre-computing the heavy input projections, SRP's forward-pass overhead is reduced to $\mathcal{O}(|\mathcal{V}| r_s (k + m + h))$. Since $r_s \ll h$ and $k, m \ll d_{in}$, this linear cost is asymptotically negligible compared to the backbone's heavy message-passing complexity of $\mathcal{O}(|\mathcal{E}|d_{in})$.

\textbf{Space Complexity.} SRP introduces a minimal set of trainable parameters. For a standard $n$-layer network, combining the initial dual-channel layer and $n-1$ projective fallback layers yields exactly $r_s(k + m + (2n-1)h) + nh + n$ trainable parameters. This scales strictly as $\mathcal{O}(nr_sh)$, avoiding input-dimension ($d_{in}$) dependencies, making it far smaller than full fine-tuning or LoRA-style baselines. Additionally, frozen buffers ($V, V_{\perp}^\top$) require only $\mathcal{O}(d_{in}(k+m) + nhk)$ memory, while activation memory remains $\mathcal{O}(|\mathcal{V}|d_{in})$, identical to the unmodified backbone.

\section{Experiments}
\label{sec:experiments}

\subsection{Experimental Setup}
\label{subsec:setup}

\paragraph{Datasets.}
We evaluate the effectiveness of SRP on ten benchmark datasets spanning both node- and graph-level classification tasks. 
For \textbf{node classification} (5-shot), we consider five widely used benchmarks, including the citation networks Cora~\citep{yang2016revisiting}, CiteSeer~\citep{yang2016revisiting}, and PubMed~\citep{yang2016revisiting}, as well as the large-scale datasets Flickr~\citep{zenggraphsaint} and ogbn-arxiv~\citep{hu2020ogb}. 
For \textbf{graph classification} (50-shot), we use five molecular and biological graph datasets from TUDataset~\citep{morris2020tudataset}: ENZYMES, DD, NCI1, NCI109, and Mutagenicity. 
These benchmarks exhibit substantial diversity in graph scale, topology, and feature dimensionality. Detailed dataset statistics are provided in Appendix~\ref{app:datasets}.

\paragraph{Pre-training strategies and Baselines.}
To evaluate the compatibility and generality of SRP across different self-supervised paradigms, we conduct experiments under two representative pre-training strategies: GraphCL~\citep{you2020graphcl} and SimGRACE~\citep{xia2022simgrace}. 
We compare SRP against ten representative baselines, including: 
Classifier Only (a linear classifier trained on frozen features), 
GPPT~\citep{sun2022gppt} (link-prediction-based prompting), 
GraphPrompt~\citep{liu2023graphprompt} (readout-level prompting), 
All-in-one~\citep{sun2023all} (multi-task prompting), 
GPF and GPF-plus~\citep{fang2023gpf} (feature-level prompting), 
EdgePrompt and EdgePrompt-plus~\citep{edgeprompt} (edge-level prompting), 
and GraphLoRA~\citep{GraphLoRA} (structure-aware low-rank adaptation). 
For fair comparison, all methods use the same backbone architecture, pre-trained checkpoints, and evaluation protocol following Prog~\citep{zi2024prog}.


\begin{table}[H]
\centering
\caption{Performance on node-level benchmarks (5-seed avg.). \textbf{Bold} = best, \underline{underline} = second best.}
\label{tab:main_results_node}
\resizebox{\textwidth}{!}{%
\begin{tabular}{llccccc}
\toprule
\textbf{Pre-training} & \textbf{Tuning Methods} & \textbf{Cora} & \textbf{CiteSeer} & \textbf{PubMed} & \textbf{ogbn-arxiv} & \textbf{Flickr}\\
\midrule
\multirow{11}{*}{GraphCL}
& Classifier Only & $53.05_{\pm 4.76}$ & $38.62_{\pm 3.43}$ & $64.28_{\pm 4.51}$ & $21.15_{\pm 1.64}$ & $24.32_{\pm 2.93}$ \\
& GPPT & $50.96_{\pm 6.67}$ & $39.50_{\pm 1.67}$ & $60.47_{\pm 4.75}$ & $17.99_{\pm 1.14}$ & $24.35_{\pm 1.84}$  \\
& GraphPrompt & $55.71_{\pm 4.62}$ & $40.81_{\pm 2.11}$ & $63.47_{\pm 2.23}$ & $21.03_{\pm 1.92}$ & \underline{$26.08_{\pm 3.44}$} \\
& ALL-in-one & $38.00_{\pm 4.17}$ & $40.27_{\pm 2.09}$ & $58.61_{\pm 3.49}$ & $16.42_{\pm 2.98}$ & $25.08_{\pm 3.44}$ \\
& GPF & $58.52_{\pm 4.07}$ & $43.55_{\pm 2.80}$ & $67.67_{\pm 3.14}$ & $21.73_{\pm 1.75}$ & $23.98_{\pm 1.71}$ \\
& GPF-plus & $52.24_{\pm 4.59}$ & $38.47_{\pm 3.27}$ & $64.30_{\pm 4.58}$ & $21.03_{\pm 1.96}$ & $25.32_{\pm 2.02}$ \\
& EdgePrompt & $58.60_{\pm 4.46}$ & $43.31_{\pm  3.23}$ & \underline{$67.76_{\pm 3.01}$} & $21.90_{\pm 1.71}$ & $24.83_{\pm 2.78}$ \\
& EdgePrompt+ & $62.88_{\pm 6.43}$ & \underline{$46.20_{\pm 0.99}$} & $67.41_{\pm 5.25}$ & $23.18_{\pm 1.26}$ & $25.57_{\pm 3.04}$ \\
& GraphLoRA & \underline{$63.35_{\pm 4.32}$} & $45.98_{\pm 3.97}$ & $67.21_{\pm 3.15}$ & \underline{$23.26_{\pm 1.89}$} & $25.71_{\pm 3.16}$ \\
& \textbf{SRP (Ours)} & $\bm{68.11_{\pm 2.29}}$ & $\bm{49.02_{\pm 2.26}}$ & $\bm{68.81_{\pm 3.92}}$ & $\bm{24.75_{\pm 1.79}}$ & $\bm{26.39_{\pm 2.11}}$ \\
\midrule
\multirow{11}{*}{SimGRACE}
& Classifier Only & $52.27_{\pm 2.74}$ & $40.45_{\pm 3.55}$ & $56.72_{\pm 3.80}$ & $20.75_{\pm 2.92}$ & $25.53_{\pm 3.98}$ \\
& GPPT & $52.07_{\pm 7.65}$ & $40.25_{\pm 3.29}$ & $58.65_{\pm 5.12}$ & $17.76_{\pm 1.80}$ & $23.37_{\pm 4.66}$ \\
& GraphPrompt & $51.42_{\pm 2.80}$ & $41.74_{\pm 2.22}$ & $55.98_{\pm 2.94}$ & $20.48_{\pm 2.57}$ & $25.88_{\pm 3.81}$ \\
& ALL-in-one & $34.64_{\pm 4.06}$ & $38.95_{\pm 2.35}$ & $54.18_{\pm 4.70}$ & $16.72_{\pm 2.90}$ & $27.68_{\pm 4.58}$ \\
& GPF & $58.23_{\pm 4.19}$ & $44.87_{\pm 4.35}$ & $61.55_{\pm 3.79}$ & $21.86_{\pm 2.91}$ & $26.51_{\pm 4.69}$ \\
& GPF-plus & $52.27_{\pm 3.34}$ & $41.02_{\pm 3.49}$ & $56.95_{\pm 3.86}$ & $21.44_{\pm 3.77}$ & $28.35_{\pm 5.50}$ \\
& EdgePrompt & $58.37_{\pm 4.51}$ & $43.94_{\pm 4.15}$ & $61.10_{\pm 3.69}$ & $21.85_{\pm 2.54}$ & $\bm{30.12_{\pm 5.04}}$ \\
& EdgePrompt+ & $62.40_{\pm 7.97}$ & \underline{$46.62_{\pm 2.53}$} & \underline{$64.91_{\pm 5.58}$} & $22.74_{\pm 2.34}$ & \underline{$28.50_{\pm 4.08}$} \\
& GraphLoRA & $\underline{63.06_{\pm 3.64}}$ & $46.39_{\pm 6.15}$ & $64.33_{\pm 5.23}$ & \underline{$23.22_{\pm 2.97}$} & $28.02_{\pm 3.54}$ \\
& \textbf{SRP (Ours)} & $\bm{67.24_{\pm 7.31}}$ & $\bm{51.37_{\pm 5.24}}$ & $\bm{67.32_{\pm 4.10}}$ & $\bm{24.01_{\pm 2.91}}$ & $28.48_{\pm 3.26}$ \\
\bottomrule
\end{tabular}%
}
\end{table}

\begin{table}[!ht]
\centering
\caption{Performance on graph-level benchmarks (5-seed avg.). \textbf{Bold} = best, \underline{underline} = second best.}
\label{tab:main_results_graph}
\resizebox{\textwidth}{!}{%
\begin{tabular}{llccccc}
\toprule
\textbf{Pre-training} & \textbf{Tuning Methods} & \textbf{ENZYMES} & \textbf{DD} & \textbf{NCI1} & \textbf{NCI109} & \textbf{Mutagenicity}\\
\midrule
\multirow{11}{*}{GraphCL}
& Classifier Only & $30.50_{\pm 1.16}$ & $62.89_{\pm 2.19}$ & $62.49_{\pm 1.95}$ & $61.68_{\pm 0.93}$ & $66.62_{\pm 1.87}$ \\
& GraphPrompt & $27.83_{\pm 1.61}$ & $64.33_{\pm 1.79}$ & $63.19_{\pm 1.71}$ & $62.18_{\pm 0.48}$ & \underline{$ 67.62_{\pm 0.65}$}\\
& ALL-in-one & $25.92_{\pm 0.55}$ & $66.54_{\pm 1.82}$ & $57.52_{\pm 2.61}$ & $62.74_{\pm 0.78}$ & $ 63.43_{\pm 2.53}$\\
& GPF & $30.08_{\pm 1.25}$ & $64.54_{\pm 2.22}$ & $62.66_{\pm 1.83}$ & $62.29_{\pm 0.90}$ & $ 66.54_{\pm 1.85}$\\
& GPF-plus & $31.00_{\pm 1.50}$ & $67.26_{\pm 2.29}$ & $64.56_{\pm 1.10}$ & $62.84_{\pm 0.22}$ & $ 66.82_{\pm 1.63}$\\
& EdgePrompt & $29.50_{\pm 1.57}$ & $64.16_{\pm 2.13}$ & $63.05_{\pm 2.11}$ & $62.59_{\pm 0.93}$ & $ 66.87_{\pm 1.88}$\\
& EdgePrompt+ & $34.00_{\pm 1.25}$ & $67.98_{\pm 2.05}$ & $66.30_{\pm 2.54}$ & \underline{$66.52_{\pm 0.91}$} & $ 67.47_{\pm 2.37}$\\
& GraphLoRA & \underline{$36.05_{\pm 1.89}$} & \underline{$68.33_{\pm 1.93}$} & \underline{$66.54_{\pm 2.71}$} & $66.40_{\pm 0.96}$ & $ 67.12_{\pm 2.11}$\\
& \textbf{SRP (Ours)} & $\bm{36.32_{\pm 2.65}}$ & $\bm{69.71_{\pm 1.72}}$ & $\bm{66.93_{\pm 1.81}}$ & $\bm{68.18_{\pm 0.55}}$ & $ \bm{69.35_{\pm 1.80}}$\\
\midrule
\multirow{11}{*}{SimGRACE}
& Classifier Only & $27.07_{\pm 1.04}$ & $61.77_{\pm 2.40}$ & $61.27_{\pm 3.64}$ & $62.12_{\pm 1.10}$ & $ 67.36_{\pm 0.71}$\\
& GraphPrompt & $26.87_{\pm 1.47}$ & $62.58_{\pm 1.84}$ & $62.45_{\pm 1.52}$ & $62.41_{\pm 0.69}$ & $ 68.03_{\pm 0.78}$\\
& ALL-in-one & $25.73_{\pm 1.18}$ & $65.16_{\pm 1.47}$ & $58.52_{\pm 1.59}$ & $62.01_{\pm 0.66}$ & $ 64.43_{\pm 1.00}$\\
& GPF & $28.53_{\pm 1.76}$ & $65.64_{\pm 0.70}$ & $61.45_{\pm 3.13}$ & $61.90_{\pm 1.26}$ & $ 67.19_{\pm 0.74}$\\
& GPF-plus & $27.33_{\pm 2.01}$ & $67.20_{\pm 1.56}$ & $61.61_{\pm 2.89}$ & $62.84_{\pm 0.23}$ & $ 67.69_{\pm 0.64}$\\
& EdgePrompt & $29.33_{\pm 2.30}$ & $63.97_{\pm 2.14}$ & $62.02_{\pm 3.02}$ & $62.02_{\pm 1.03}$ & $ 67.55_{\pm 0.85}$\\
& EdgePrompt+ & $32.67_{\pm 2.53}$ & \underline{$67.72_{\pm 1.62}$} & $\bm{67.07_{\pm 1.96}}$ & \underline{$66.53_{\pm 1.30}$} & \underline{$68.31_{\pm 1.36}$}\\
& GraphLoRA & \underline{$35.17_{\pm 1.42}$} & $67.02_{\pm 1.97}$ & \underline{$66.83_{\pm 2.39}$} & $65.33_{\pm 1.04}$ & $ 67.58_{\pm 1.67}$\\
& \textbf{SRP (Ours)} & $\bm{35.38_{\pm 1.88}}$ & $\bm{68.12_{\pm 1.97}}$ & $66.53_{\pm 2.94}$ & $\bm{67.07_{\pm 0.83}}$ & $\bm{68.92_{\pm 2.73}}$\\
\bottomrule
\end{tabular}%
}
\end{table}

\subsection{Main Results}
We compare the node-level and graph-level classification accuracy  of our methods and other baselines. Table~\ref{tab:main_results_node} and Table~\ref{tab:main_results_graph} report results under GraphCL and SimGRACE pre-training.
We make several key observations.
First, SRP consistently achieves the best or second best performance across almost all dataset.
Second, the gains are most pronounced on node classification benchmarks with high input dimension, where the null-space channel can extract substantial discriminative signal from the frozen encoder.
Third, on graph classification benchmarks where the input dimension is small, SRP operates via its projective fallback (Section~\ref{subsec:fusion}), yet remains superior to the strongest baselines, confirming that the spectral reversal mask provides meaningful gains.
These results validate the central thesis of this work: inverting the spectral priority of frozen GNN weights is a broadly effective strategy for graph prompt tuning.

\subsection{Ablation Study}
\label{subsec:ablation}
To disentangle the contributions of different components in SRP, we evaluate four ablation variants under both GraphCL and SimGRACE pre-training settings:
\textbf{SRP-NM} removes the null-space channel,
\textbf{SRP-NS} removes the spectral channel,
\textbf{SRP-NR} disables spectral reversal by adopting uniform weighting across all singular directions (i.e., $\mu_j \equiv 1$),
and \textbf{SRP-Bi} replaces the reverse weighting strategy with a bidirectional mask that amplifies both dominant and trailing spectral directions. 
The results in Table~\ref{tab:ablation} reveal a consistent trend across both pre-training frameworks.

\textbf{Effectiveness of the spectral and null-space channels.}
Both the spectral channel and the null-space channel play essential roles in extracting discriminative information for downstream adaptation. 
Removing either component leads to clear performance degradation. 
In particular, \textbf{SRP-NM} and \textbf{SRP-NS} consistently underperform the full SRP model across all datasets and pre-training settings, demonstrating that both weak spectral directions and null-space information contribute substantially to effective prompt learning.

\textbf{Effectiveness of the spectral reversal principle.}
SRP improves adaptation by reversing the pretrained spectral bias and emphasizing weak singular directions. 
When this mechanism is removed (\textbf{SRP-NR}) or replaced with bidirectional amplification (\textbf{SRP-Bi}), performance consistently declines relative to the full SRP model. 
These results suggest that selectively emphasizing trailing singular directions is more effective than uniformly weighting or jointly amplifying both spectral extremes for few-shot graph adaptation.

\begin{table*}[t]
    \centering
    \caption{Ablation Study of Different Variants on GraphCL and SimGRACE.}
    \label{tab:ablation}
    \renewcommand{\arraystretch}{1.2}
    \small
    \resizebox{0.85\textwidth}{!}{%
    \begin{tabular}{lcccccc}
        \toprule
        \multirow{2}{*}{\textbf{Variant}} & \multicolumn{3}{c}{\textbf{GraphCL}} & \multicolumn{3}{c}{\textbf{SimGRACE}} \\
        \cmidrule(lr){2-4} \cmidrule(lr){5-7}
        & \textbf{Cora} & \textbf{CiteSeer} & \textbf{Avg. $\Delta$} & \textbf{Cora} & \textbf{CiteSeer} & \textbf{Avg. $\Delta$} \\
        \midrule
        \textbf{SRP} (full)              & $\bm{68.1_ {\pm 2.3}}$ & $\bm{49.0_ {\pm 2.3}}$ & $-0.0$ & $\bm{67.2_{\pm 7.3}}$ & $\bm{51.4_{\pm 5.2}}$ & $-0.0$ \\
        SRP-NM (w/o null channel)        & $63.4_ {\pm 2.8}$ & $41.1_{\pm 3.8}$ & $-6.3$ & $60.3_ {\pm 5.9}$ & $43.5_ {\pm 2.9}$ & $-7.4$ \\
        SRP-NS (w/o spectral channel)        & $64.4_ {\pm 2.1}$ & $45.6_ {\pm 3.2}$ & $-3.6$ & $63.5_ {\pm 6.6}$ & $47.7_ {\pm 4.3}$ & $-3.7$ \\
        SRP-NR (w/o reverse)             & $64.7_ {\pm 2.7}$ & $48.0 _{\pm 4.0}$ & $-2.2$ & $63.4_ {\pm 4.4}$ & $49.9_ {\pm 4.9}$ & $-2.7$ \\
        SRP-Bi (bidirectional)           & $64.7_ {\pm 3.3}$ & $47.9 _{\pm 3.3}$ & $-2.3$ & $63.0 _{\pm 4.5}$ & $50.0_ {\pm 4.8}$ & $-2.8$ \\
        \bottomrule
    \end{tabular}%
    }
\end{table*}



\FloatBarrier 

\subsection{Additional Evaluation}
\label{subsec:analysis}
\textbf{Label sensitivity}
We analyzed the impact of the number of labels for each class during prompt tuning in downstream tasks, as shown in Figure \ref{fig:fewshot_combined}. We found that the performance of node classification tasks fluctuates significantly with changes in the number of samples when the shot number is extremely few. In contrast, graph classification tasks are less affected by the number of labels.\\

\textbf{Parameter sensitivity.}
We study the sensitivity of SRP to its two key hyperparameters: the null-space PCA dimension $m$ and the adaptor bottleneck rank $r_s$ and numbers of labels. $m$ controls how many low-amplification singular directions are used to construct the spectral reparameterization subspace, $r_s$ determines the capacity of the lightweight adapter, shot number reflects the ability to fine tune the model with a limited number of labels.
Figure~\ref{fig:fewshot_combined} reports 5-shot node classification accuracy under GraphCL as we vary $r_s \in \{16, 32, 64, 128\}$ and $m \in \{4, 8, 16, 32\}$ independently.
Performance remains stable across a wide range of values, with accuracy plateauing at the default settings ($r_s{=}32$, $m{=}16$), confirming that SRP is robust to the choice of these hyperparameters and does not require extensive tuning in practice.  \\

\textbf{Convergence Speed.} 
We also investigate the convergence speed of our method compared with other strong baselines. Figure~\ref{fig:convergence} illustrates the accuracy curves of our method and the baselines under two pre-training strategies. According to Figure~\ref{fig:convergence}, we can observe that SRP can generally converge faster than other methods, which means extra efficiency by adopting fewer epoch numbers in tuning the whole pre-trained models.
\begin{figure}[h]
  \centering
  \includegraphics[width=\textwidth]{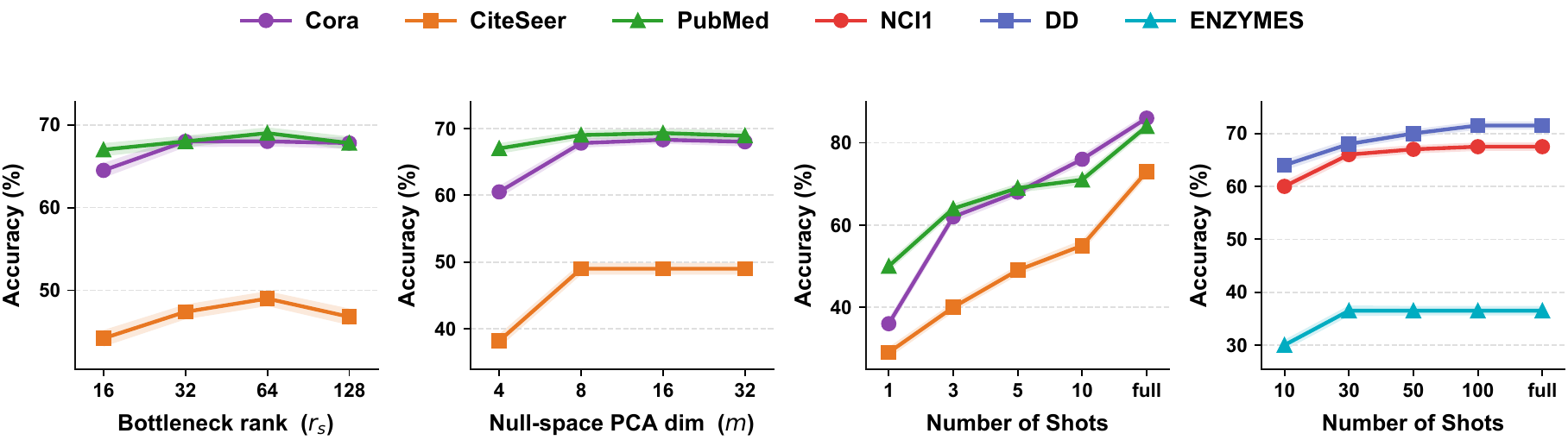}
  \caption{Label Sensitivity and Parameter Sensitivity of NS-PCA and bottleneck rank.}
  \label{fig:fewshot_combined}
\end{figure}

\begin{figure}[h]
  \centering
  \includegraphics[width=\textwidth]{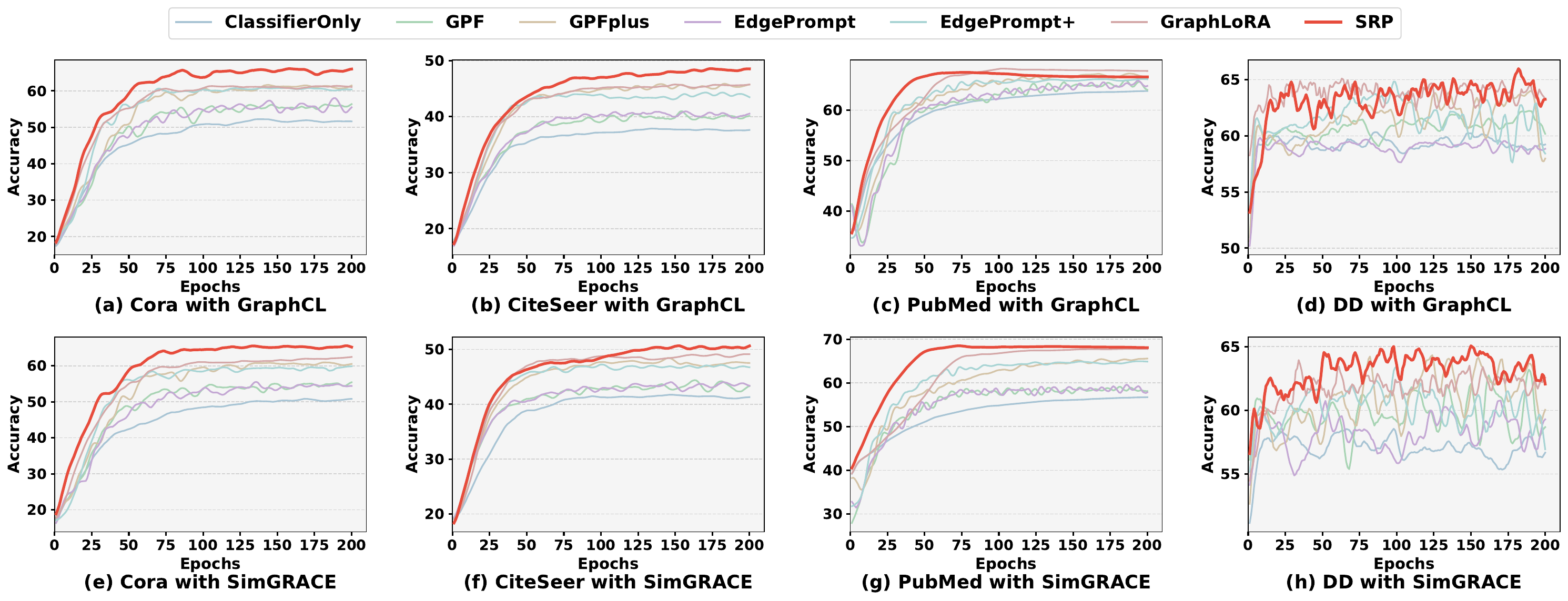}
  \caption{Convergence speed of different methods.}
  \label{fig:convergence}
\end{figure}


\section{Conclusion}
\label{sec:conclusion}

In this work, we identified \emph{spectral bias} as a key bottleneck in adapting pre-trained GNNs, where self-supervised objectives overemphasize dominant singular directions while underutilizing lower-energy components. Using the DPMS, we empirically and theoretically showed that downstream utility is often concentrated in these underexplored spectral directions. To address this misalignment, we introduced SRP, a parameter-efficient framework that rebalances spectral contributions of frozen encoders via soft-thresholding and null-space augmentation. By reallocating capacity toward underutilized components, SRP achieves state-of-the-art performance across diverse benchmarks while using less than 7\% of backbone parameters. Our results highlight a simple but important principle: in parameter-efficient adaptation, \emph{which directions are adapted} can matter more than how many parameters are introduced.

\textbf{Limitations and Future Works.}
While SRP remains robustness across different graph homophily levels and backbone architectures according to our additional results~\ref{sec:additional_analysis}, its spectral reweighting currently relies on a shared singular basis and a global threshold, which may not fully capture region-specific frequency heterogeneity. In addition, the effectiveness of the null-space channel depends on the relationship between input dimensionality and encoder rank, and can be limited when the frozen transformation has little or no null space. Future work includes instance- or structure-adaptive spectral reweighting driven by downstream feedback, as well as extending the spectral perspective to multi-modal graph representation learning.


\begin{ack}
 This work was supported by the National Natural Science
Foundation of China (No. 62276053); Center for HPC, University of Electronic Science and Technology of China; Sichuan Province Science and Technology Support Program (No. 2025ZDZX0016).
\end{ack}


\bibliographystyle{unsrt}
\bibliography{ref}

\newpage
\clearpage 
\appendix

\section{Dataset Statistics}
\label{app:datasets}

Tables~\ref{tab:dataset_node} and~\ref{tab:dataset_graph} summarize the basic statistics of all ten benchmark datasets used in our main experiments and three benchmark in our extension to heterophilic graphs~\cite{Pei2020Geom-GCN:}.

\begin{table}[h]
  \caption{Node-level benchmark datasets. \#Feat.\ is the raw input feature dimension; \#Classes is the number of downstream classes.}
  \label{tab:dataset_node}
  \centering\small
  \begin{tabular}{lrrrrr}
    \toprule
    \textbf{Dataset} & \textbf{\#Nodes} & \textbf{\#Edges} & \textbf{\#Feat.} & \textbf{\#Classes} & \textbf{Task} \\
    \midrule
    Cora       & 2,708   & 10,556    & 1,433 & 7  & Node \\
    CiteSeer   & 3,327   & 9,104     & 3,703 & 6  & Node \\
    PubMed     & 19,717  & 88,648    & 500   & 3  & Node \\
    ogbn-arxiv & 169,343 & 1,166,243 & 128   & 40 & Node \\
    Flickr     & 89,250  & 899,756   & 500   & 7  & Node \\
    Cornell    & 183     & 298       & 1703  & 5
    & Node \\
    Squirrel    & 5201     & 217073   & 2089  & 5
    & Node \\
    Chameleon   & 2277   & 36101    & 2277  & 5
    & Node \\
    \bottomrule
  \end{tabular}
\end{table}

\begin{table}[h]
  \caption{Graph-level benchmark datasets (TUDataset). Avg.\ Nodes and Avg.\ Edges are computed over the graph collection.}
  \label{tab:dataset_graph}
  \centering\small
  \begin{tabular}{lrrrrrr}
    \toprule
    \textbf{Dataset} & \textbf{\#Graphs} & \textbf{Avg.\ Nodes} & \textbf{Avg.\ Edges} & \textbf{\#Feat.} & \textbf{\#Classes} & \textbf{Task} \\
    \midrule
    ENZYMES      & 600   & 32.63  & 124.27   & 3  & 6 & Graph \\
    DD           & 1,178 & 284.32 & 1,431.32 & 89 & 2 & Graph \\
    NCI1         & 4,110 & 29.87  & 64.60    & 37 & 2 & Graph \\
    NCI109       & 4,127 & 29.68  & 64.26    & 38 & 2 & Graph \\
    Mutagenicity & 4,337 & 30.32  & 61.54    & 14 & 2 & Graph \\ 
    \bottomrule
  \end{tabular}
\end{table}

Our next objective is to investigate how variations in curvature~\cite{wang2026postgcndecaderevisited}, an important structural property in graph representation learning, may affect the performance of our method.

\section{Proofs of Theoretical Results}
\label{sec:app_proofs}

In this section, we provide the rigorous mathematical proofs for the theoretical claims presented in Section \ref{sec:spectral_bias}.

\subsection{Proof of Proposition \ref{prop:spectral_align} (Spectral Concentration)}

\textbf{Statement:} Given the rank-1 update $\Delta W^{(t)} = -\eta_t\, g^{(t)}(h^{(t)})^\top$ and the spectral alignment condition $\mathbb{E}_t\big[\|(I - P_r)h^{(t)}\|^2 / \|h^{(t)}\|^2\big] \leq \delta^2$, the cumulative gradient update satisfies:
\begin{equation*}
  \frac{\mathbb{E}\big[\|P_r\,\Delta W^\top\|_F^2\big]}{\mathbb{E}\big[\|\Delta W\|_F^2\big]} \;\geq\; 1 - \delta^2.
\end{equation*}

\textit{Proof.} 
Consider the transpose of the gradient update at a single step $t$:
\begin{equation}
    (\Delta W^{(t)})^\top = -\eta_t\, h^{(t)}(g^{(t)})^\top.
\end{equation}
We can decompose the input vector $h^{(t)}$ into two orthogonal components using the projection matrix $P_r$ and its orthogonal complement $(I - P_r)$:
\begin{equation}
    h^{(t)} = P_r h^{(t)} + (I - P_r) h^{(t)}.
\end{equation}
The Frobenius norm squared of the update transposed is determined by the outer product of these vectors. Utilizing the property $\|a b^\top\|_F^2 = \|a\|_2^2 \|b\|_2^2$, we have:
\begin{equation}
    \|(\Delta W^{(t)})^\top\|_F^2 = \eta_t^2 \|h^{(t)}\|_2^2 \|g^{(t)}\|_2^2.
\end{equation}
Similarly, the energy of the update residing outside the top-$r$ principal subspace is:
\begin{equation}
    \|(I - P_r)(\Delta W^{(t)})^\top\|_F^2 = \eta_t^2 \|(I - P_r)h^{(t)}\|_2^2 \|g^{(t)}\|_2^2.
\end{equation}
Assuming that the magnitude of the upstream gradient $\|g^{(t)}\|_2^2$ is independent of the normalized directional projection of $h^{(t)}$ over the training trajectory, we can take the expectation over $t$:
\begin{equation}
    \mathbb{E}\big[\|(I - P_r)\Delta W^\top\|_F^2\big] = \mathbb{E}\left[ \frac{\|(I - P_r)h^{(t)}\|_2^2}{\|h^{(t)}\|_2^2} \cdot \eta_t^2 \|h^{(t)}\|_2^2 \|g^{(t)}\|_2^2 \right].
\end{equation}
Given the spectral alignment condition $\mathbb{E}_t\big[\|(I - P_r)h^{(t)}\|^2 / \|h^{(t)}\|^2\big] \leq \delta^2$, it directly follows that:
\begin{equation}
    \mathbb{E}\big[\|(I - P_r)\Delta W^\top\|_F^2\big] \leq \delta^2\, \mathbb{E}\big[\|\Delta W^\top\|_F^2\big] = \delta^2\, \mathbb{E}\big[\|\Delta W\|_F^2\big].
\end{equation}
By the Pythagorean theorem for the Frobenius norm, $\|P_r\Delta W^\top\|_F^2 + \|(I - P_r)\Delta W^\top\|_F^2 = \|\Delta W\|_F^2$. Substituting the upper bound yields the final concentration result:
\begin{equation}
    \mathbb{E}\big[\|P_r\,\Delta W^\top\|_F^2\big] \geq (1 - \delta^2)\, \mathbb{E}\big[\|\Delta W\|_F^2\big],
\end{equation}
which concludes the proof. 

\vspace{1em}

\subsection{Proof of Theorem \ref{thm:reversal} (Reachability and Approximation Limits)}

\textbf{Statement:} For a frozen encoder $W = U\Sigma V^\top$ and an optimal downstream classifier $W_\text{down}$, the minimum approximation error using any linear head $A$ is $\varepsilon^2(W) = \|P_\mathcal{N} W_\text{down}\|_F^2$, achieved at $A^* = U\Sigma^{-1}V^\top W_\text{down}$.

\textit{Proof.}
The objective is to minimize the approximation error $\mathcal{L}(A) = \|W_\text{down} - W^\top A\|_F^2$ over all possible downstream projection heads $A \in \mathbb{R}^{d_\text{out} \times C}$.
Substituting the compact SVD of the transposed encoder $W^\top = V\Sigma U^\top$ into the objective yields:
\begin{equation}
    \mathcal{L}(A) = \|W_\text{down} - V\Sigma U^\top A\|_F^2.
\end{equation}
Since the columns of $V$ are orthonormal, $VV^\top$ forms the orthogonal projection onto $\text{col}(V)$, and $P_\mathcal{N} = I - VV^\top$ is the projection onto the null space $\ker(W)$. We can orthogonally decompose the target matrix $W_\text{down}$ as:
\begin{equation}
    W_\text{down} = VV^\top W_\text{down} + P_\mathcal{N} W_\text{down}.
\end{equation}
Because the term $V\Sigma U^\top A$ lies entirely within the column space $\text{col}(V)$, it is orthogonal to $P_\mathcal{N} W_\text{down}$. By applying the Pythagorean theorem for the Frobenius norm, we can decouple the objective function:
\begin{equation}
    \mathcal{L}(A) = \|VV^\top W_\text{down} - V\Sigma U^\top A\|_F^2 + \|P_\mathcal{N} W_\text{down}\|_F^2.
\end{equation}
The second term, $\|P_\mathcal{N} W_\text{down}\|_F^2$, is entirely independent of $A$ and represents the irreducible error (the fundamental blind spot). To minimize $\mathcal{L}(A)$, we must force the first term to zero:
\begin{equation}
    VV^\top W_\text{down} = V\Sigma U^\top A.
\end{equation}
Left-multiplying both sides by $V^\top$ (noting that $V^\top V = I$) gives:
\begin{equation}
    V^\top W_\text{down} = \Sigma U^\top A.
\end{equation}
Since $\Sigma$ is a diagonal matrix of strictly positive singular values, it is invertible. We can explicitly solve for $A$ by left-multiplying by $U\Sigma^{-1}$:
\begin{equation}
    A^* = U\Sigma^{-1}V^\top W_\text{down}.
\end{equation}
Substituting $A^*$ back into $\mathcal{L}(A)$ eliminates the first term, formally proving that the minimum reachable error is exactly:
\begin{equation}
    \varepsilon^2(W) = \min_{A} \mathcal{L}(A) = \|P_\mathcal{N} W_\text{down}\|_F^2.
\end{equation}
This establishes both the optimal projection head structure and the irreducible nature of downstream targets residing in the null space.

\subsection{Null-Space PCA Optimality (Proof of Theorem~\ref{thm:nspca})}

\begin{theorem}[Optimality of Null-Space PCA]
\label{thm:nspca}
Let $\Sigma_X = \EE[xx^\top]$ be the input feature covariance and $P_\calN = I - VV^\top$ the orthogonal projector onto $\ker(W)$.
Among all orthonormal matrices $N_o \in \RR^{d_\text{in} \times m}$ satisfying $WN_o = 0$, the null-space PCA basis uniquely solves:
\begin{equation}
N_o^*
=
\operatorname*{arg\,max}_{N_o}
\left\{
    \operatorname{tr}\!\left(N_o^\top \Sigma_X N_o\right)
    \;\middle|\;
    N_o^\top N_o = I_m,\;
    WN_o = 0
\right\}.
\label{eq:nspca_opt}
\end{equation}
Standard PCA (unconstrained) is generically infeasible for~\eqref{eq:nspca_opt} and therefore suboptimal.
\end{theorem}

\begin{proof}
The constraint $WN_o = 0$ is equivalent to $V^\top N_o = 0$ (since $\ker(W) = \ker(V^\top)$).  Writing $N_o = P_\calN C$ for an arbitrary matrix $C$, the constraint $V^\top N_o = 0$ is automatically satisfied, and~\eqref{eq:nspca_opt} reduces to:
\[
  \max_{\substack{C^\top P_\calN C = I_m}} \mathrm{tr}(C^\top P_\calN \Sigma_X P_\calN C).
\]
By the Ky Fan maximum theorem~\citep{horn2013matrix}, the maximum equals $\sum_{j=1}^m \lambda_j(P_\calN \Sigma_X P_\calN)$, attained uniquely by the top-$m$ eigenvectors of $P_\calN \Sigma_X P_\calN$---the null-space PCA basis.  An unconstrained PCA solution lies generically outside $\ker(W)$ (a measure-zero coincidence would be needed for it to satisfy $V^\top N_o = 0$), so it is infeasible and thus suboptimal.
\end{proof}

\subsection{PAC-Bayes Generalization Bound (Proof of Theorem~\ref{thm:pac})}

\begin{theorem}[PAC-Bayes Generalization Bound]
\label{thm:pac}
Let $\Theta$ denote SRP's learnable parameters with $d_\Theta = |\calV_w| r_s + m r_s + r_s d_\text{out} + d_\text{out} + 1$.
Let $\hat{R}(\Theta)$ be the empirical risk on $n$ labeled examples, $R(\Theta)$ the expected risk, and $\delta \in (0,1)$.
For any prior $\pi$ and posterior $\rho$ concentrated at $\hat\Theta$, with probability $\geq 1-\delta$:
\begin{equation}
  R(\hat\Theta) \;\leq\; \hat{R}(\hat\Theta) + \sqrt{\frac{\mathrm{KL}(\rho \| \pi) + \log(2\sqrt{n}/\delta)}{2n}}.
  \label{eq:pac_general}
\end{equation}
Setting $\pi = \calN(0, \kappa^2 I_{d_\Theta})$ and $\rho = \calN(\hat\Theta, \hat\sigma^2 I_{d_\Theta})$ with $\kappa = \kappa(W) := \sigma_\text{max}/\sigma_\text{min}$, one obtains:
\begin{equation}
  R(\hat\Theta) \;\leq\; \hat{R}(\hat\Theta) + \sqrt{\frac{|\calV_w| r_s \log\kappa(W) + \log(2n/\delta)}{2n}},
  \label{eq:pac_srp}
\end{equation}
where $d_\Theta = O(|\calV_w| r_s)$ is the dominant parameter count (spectral adaptor $A \in \RR^{|\calV_w| \times r_s}$).
\end{theorem}

\begin{proof}
Equation~\eqref{eq:pac_general} is McAllester's PAC-Bayes bound~\citep{mcallester1999some}.
For the chosen Gaussian prior and posterior, the KL divergence evaluates to:
\begin{equation}
  \mathrm{KL}(\rho\|\pi)
  = \frac{1}{2}\!\left[d_\Theta\log\frac{\kappa^2}{\hat\sigma^2}
    - d_\Theta
    + \frac{d_\Theta\hat\sigma^2 + \|\hat\Theta\|^2}{\kappa^2}\right].
  \label{eq:kl_bound}
\end{equation}
The effective signal scale per direction satisfies $\mu_j/\sigma_j \leq \mu_j/\sigma_\text{min} \leq 1/\sigma_\text{min}$, so $\|\hat\Theta\|^2 = O(d_\Theta / \sigma_\text{min}^2)$.
Setting $\kappa = \sigma_\text{max}/\sigma_\text{min}$ and optimizing $\hat\sigma^2$ yields $\mathrm{KL}(\rho\|\pi) = O(d_\Theta \log\kappa(W))$.
Substituting into~\eqref{eq:pac_general} with $d_\Theta \approx |\calV_w| r_s$ gives~\eqref{eq:pac_srp}.

\textit{Complexity interpretation.}  The bound depends on $|\calV_w|$ (number of under-served directions targeted by SRP) rather than on $k$ (total singular directions) or $d_\text{in}d_\text{out}$ (full-weight dimension).  Since DPMS-guided prompting restricts capacity to the weak subspace, the effective model complexity is $|\calV_w| r_s \ll k r_s \ll d_\text{in}d_\text{out}$, directly linking the DPMS-based design philosophy to a tighter generalization bound.
\end{proof}

\subsection{Algebraic Equivalence of the Projective Fallback}
\label{subsec:appendix_fallback}

In our SRP architecture, we introduce a projective fallback path when the input dimension is small, specifically when $d_\text{in} \leq k + m$. In this section, we rigorously prove that this fallback is not a heuristic deviation, but rather an exact algebraic simplification of the dual-channel architecture.

\begin{theorem}[Algebraic Isomorphism of the Fallback]
\label{thm:fallback_equivalence}
Let the pre-trained weight have rank $k$, and let $V \in \mathbb{R}^{d_\text{in} \times k}$ be its right singular vectors. Let $\mu \in \mathbb{R}^k$ denote the soft-threshold mask. When $d_\text{in} - k \leq m$, the projective fallback parameterized by a single adaptor $P \in \mathbb{R}^{d_\text{in} \times r_s}$ acting on $h_\text{weak}$ spans the exact same hypothesis space as the dual-channel prompt parameterized by $A_s \in \mathbb{R}^{k \times r_s}$ and $A_n \in \mathbb{R}^{(d_\text{in}-k) \times r_s}$.
\end{theorem}

\begin{proof}
First, we analyze the structure of the dual-channel prompt. The hidden representation $h$ can be orthogonally decomposed into the reachable subspace $h_\parallel = hVV^\top$ and the exact null space $h_\perp = h(I - VV^\top) = hP_\mathcal{N}$. 
In the dual-channel mode, the target response $Z_\text{dual}$ is the sum of the masked spectral channel and the null-space channel:
\begin{equation}
    Z_\text{dual} = (hV \mathrm{diag}(\mu)) A_s + (h V_{\perp}^\top) A_n,
    \label{eq:fallback_1}
\end{equation}
where $V_{\perp}^\top$ acts as the PCA basis matrix of the null space. Crucially, when $d_\text{in} - k \leq m$, the required PCA dimension $m$ exceeds or equals the dimension of the entire null space. Thus, the columns of $V_{\perp}^\top$ simply form a complete orthonormal basis for $P_\mathcal{N}$, meaning $V_{\perp}^\top V_{\perp} = P_\mathcal{N}$.

Next, we analyze the input to the projective fallback, defined as $h_\text{weak} = h - (hV \mathrm{diag}(1-\mu))V^\top$. We can expand $h$ using the orthogonal decomposition $h = hVV^\top + hP_\mathcal{N}$:
\begin{align}
    h_\text{weak} &= (hVV^\top + hP_\mathcal{N}) - hV \mathrm{diag}(1-\mu) V^\top \nonumber \\
    &= hV \big(I - \mathrm{diag}(1-\mu)\big) V^\top + hP_\mathcal{N} \nonumber \\
    &= hV \mathrm{diag}(\mu) V^\top + hP_\mathcal{N}.
    \label{eq:fallback_2}
\end{align}
Equation \eqref{eq:fallback_2} reveals a critical property: $h_\text{weak}$ flawlessly preserves both the masked weak spectral signals and the entirety of the null-space signals, seamlessly superimposing them in the original ambient space $\mathbb{R}^{d_\text{in}}$.

Applying the single linear adaptor $P$ to $h_\text{weak}$ yields:
\begin{equation}
    Z_\text{fallback} = h_\text{weak} P = \big(hV \mathrm{diag}(\mu)\big) (V^\top P) + (h P_\mathcal{N}) P.
    \label{eq:fallback_3}
\end{equation}
To prove the isomorphism, we must show that there exists a matrix $P$ such that $Z_\text{fallback} = Z_\text{dual}$ for any arbitrary $A_s$ and $A_n$. By comparing Eq. \eqref{eq:fallback_3} with Eq. \eqref{eq:fallback_1}, and using the identity $h P_\mathcal{N} P = h V_{\perp}^\top (V_{\perp} P)$, this equivalence requires:
\begin{equation}
    V^\top P = A_s \quad \text{and} \quad V_{\perp} P = A_n.
    \label{eq:fallback_4}
\end{equation}
Since $V$ and $V_{\perp}^\top$ span mutually orthogonal complementary subspaces, the concatenated matrix $[V, V_{\perp}^\top] \in \mathbb{R}^{d_\text{in} \times d_\text{in}}$ is an orthogonal square matrix, and thus strictly invertible. Consequently, we can uniquely solve for $P$:
\begin{equation}
    P = V A_s + V_{\perp}^\top A_n.
    \label{eq:fallback_5}
\end{equation}
Equation \eqref{eq:fallback_5} demonstrates that optimizing the single matrix $P$ over $h_\text{weak}$ is mathematically identical to independently optimizing $A_s$ and $B_n$ in their respective sub-manifolds. Therefore, when $d_\text{in} \leq k + m$, the projective fallback explicitly circumvents the redundant PCA calculations while strictly maintaining the theoretical capacity and spectral-reversal philosophy of SRP.
\end{proof}

\section{Additional Experiments and Analysis}
\label{sec:additional_analysis}

\subsection{Generalization to GraphGPS and Graphormer}
\label{sec:graph_transformers}

To evaluate whether SRP extends beyond conventional message-passing GNNs, we replace the default backbone with GraphGPS and Graphormer and evaluate on two node-level datasets (Cora and CiteSeer) and two graph-level datasets (NCI1 and ENZYMES). We follow the same pre-training objective, data split, few-shot setting, training budget, and evaluation protocol as in the main experiments. Classifier-only, GraphLoRA, and SRP use the same frozen encoder and downstream splits. We retain the default SRP configuration with adaptor dimension $r=32$ and PCA dimension $m=16$. Results are averaged over five random seeds.

\begin{table}[ht]
\centering
\small
\caption{Performance of SRP with GraphGPS and Graphormer backbones. Results are reported as mean $\pm$ standard deviation over five runs.}
\label{tab:graph_transformer}
\begin{tabular}{llccc}
\toprule
Backbone & Dataset & Classifier-only & GraphLoRA & SRP \\
\midrule
GraphGPS & Cora     & $48.17\pm6.12$ & $64.33\pm6.88$ & $\mathbf{66.62\pm5.27}$ \\
GraphGPS & CiteSeer & $39.13\pm4.57$ & $46.80\pm3.38$ & $\mathbf{58.74\pm5.01}$ \\
GraphGPS & NCI1     & $63.32\pm1.51$ & $62.49\pm1.93$ & $\mathbf{66.81\pm2.13}$ \\
GraphGPS & ENZYMES  & $35.08\pm2.82$ & $45.00\pm3.25$ & $\mathbf{47.75\pm2.71}$ \\
\midrule
Graphormer & Cora     & $51.13\pm4.04$ & $62.88\pm5.49$ & $\mathbf{65.69\pm5.97}$ \\
Graphormer & CiteSeer & $39.01\pm1.74$ & $47.13\pm1.72$ & $\mathbf{54.89\pm1.60}$ \\
Graphormer & NCI1     & $56.75\pm2.76$ & $\mathbf{60.54\pm1.03}$ & $59.72\pm2.75$ \\
Graphormer & ENZYMES  & $26.75\pm2.41$ & $\mathbf{39.67\pm2.49}$ & $38.92\pm2.13$ \\
\bottomrule
\end{tabular}
\end{table}

SRP improves over Classifier-only in all eight settings, with an average absolute gain of $12.48$ percentage points. It achieves the best result in six settings and remains within $0.82$ percentage points of GraphLoRA in the other two. These results indicate that spectral rebalancing is not specific to standard message-passing GNNs and can also be applied to hybrid local--global and attention-based graph encoders.

\subsection{Role and Sensitivity of the Temperature Coefficient}
\label{sec:gamma_sensitivity}

Recall that the spectral mask is defined as
\begin{equation}
    \mu_j = \operatorname{sigmoid}\!\left((\tau\sigma_{\max}-\sigma_j)\gamma\right),
\end{equation}
where $\tau$ determines the spectral cutoff and $\gamma$ controls the sharpness of the transition around it. The interval $\mu_j\in[0.1,0.9]$ corresponds to a transition width of approximately $4.39/\gamma$ in singular-value units. Thus, $\gamma$ controls the smoothness of the partition without changing the cutoff determined by $\tau$.

We evaluate $\gamma$ from an extremely soft mask to an approximately discrete partition. Results are reported as mean accuracy over five seeds.

\begin{table}[ht]
\centering
\small
\caption{Sensitivity to the temperature coefficient $\gamma$.}
\label{tab:gamma_sensitivity}
\begin{tabular}{lcccc}
\toprule
Mask setting & Cora & CiteSeer & NCI1 & ENZYMES \\
\midrule
$\gamma=0.001$ & 64.55 & 47.64 & 66.63 & 36.42 \\
$\gamma=0.1$   & 64.88 & 48.03 & 66.52 & \textbf{36.92} \\
$\gamma=1$     & 67.80 & 48.81 & \textbf{66.96} & 36.50 \\
$\gamma=10$    & \textbf{68.10} & \textbf{48.99} & 66.80 & 36.67 \\
$\gamma=100$   & 68.06 & 48.87 & 66.84 & 36.25 \\
Hard threshold (STE) & 67.95 & 48.90 & 66.89 & 35.67 \\
\bottomrule
\end{tabular}
\end{table}

At $\gamma=0.001$, the mask is nearly constant around $0.5$, with entropy close to one bit, and therefore provides little distinction between strong and weak directions. Increasing $\gamma$ to $10$ improves accuracy by $3.55$ points on Cora and $1.35$ points on CiteSeer. For $\gamma\geq1$, performance becomes stable. At $\gamma=10$, the mask entropy decreases to $0.004$--$0.064$ bits, and only $0.3\%$--$3.6\%$ of directions remain in the transition region. Further sharpening therefore affects only a small number of boundary directions. We use $\gamma=10$ for all datasets without dataset-specific tuning.

To separate spectral selection from generic dimensionality reduction, we further construct masks with exactly the same cardinality as the learned $\gamma=10$ mask, while selecting the strongest, weakest, or random singular directions.

\begin{table}[ht]
\centering
\small
\caption{Matched-cardinality comparison of different spectral selections.}
\label{tab:mask_selection}
\begin{tabular}{lccccc}
\toprule
Dataset & Strong-$k$ & Random-$k$ & Weak-$k$ & Hard (STE) & Soft $\gamma=10$ \\
\midrule
Cora     & 64.25 & 64.44 & 67.80 & 67.95 & \textbf{68.10} \\
CiteSeer & 47.10 & 47.46 & 48.78 & 48.90 & \textbf{48.99} \\
NCI1     & 66.42 & 66.45 & 66.87 & \textbf{66.89} & 66.80 \\
\bottomrule
\end{tabular}
\end{table}

Weak-$k$ consistently outperforms Strong-$k$ and Random-$k$, particularly on Cora and CiteSeer, showing that the gain is not explained by dimensionality reduction alone. The similar performance of the soft and hard masks is also expected when few singular values lie near the cutoff. The sigmoid mask provides a differentiable relaxation during optimization while becoming nearly discrete after convergence.

\subsection{Extension to Heterophilic Graphs}
\label{sec:heterophily}

Although the main benchmarks are predominantly homophilic, the SRP architecture is not restricted to homophilic graphs. The connection can be understood from the spectral filtering induced by message passing. Consider
\begin{equation}
    Z=\widehat A W^\top H,
    \qquad
    W=U\Sigma V^\top,
\end{equation}
and let the normalized graph Laplacian be
\begin{equation}
    L=I-\widehat A_{\mathrm{sym}}
      =\Phi\Lambda\Phi^\top .
\end{equation}
For a graph signal $s$, its graph frequency can be measured by its normalized Dirichlet energy,
\begin{equation}
    \operatorname{freq}(s)
    =\frac{s^\top Ls}{s^\top s}.
\end{equation}
Since $\widehat A_{\mathrm{sym}}\phi_i=(1-\lambda_i)\phi_i$, repeated propagation scales the $i$-th graph Fourier component by $(1-\lambda_i)^K$ after $K$ propagation steps. Standard message passing therefore favors low-frequency components while attenuating high-frequency variation.

During pre-training, the features used to update the weights have already passed through this propagation operator. Consequently, the feature covariance observed by optimization is biased toward graph-frequency components preserved by message passing. This provides a graph-specific connection between the graph spectrum and the singular spectrum analyzed in our method.

SRP introduces its prompt after neighborhood aggregation:
\begin{equation}
    H^{(\ell+1)}
    =
    \phi\!\left(
    \widehat A W^{(\ell)\top}H^{(\ell)}
    +p^{(\ell)}(H^{(\ell)})
    \right).
\end{equation}
The prompt $p^{(\ell)}$ is therefore not filtered again by the same $\widehat A$ in the current layer, providing a complementary pathway for node-local and high-frequency information suppressed by the frozen message-passing branch.

We further evaluate SRP on three heterophilic datasets using GraphCL pre-training and the same hyperparameter configuration as in the main experiments. Across both homophilic and heterophilic graph data, prior approaches typically rely on carefully designed filters to balance the learned graph representations~\cite{li2024pcconv}, whereas our method demonstrates robustness to varying levels of heterophily without requiring such specialized designs.

\begin{table}[ht]
\centering
\small
\caption{Performance on heterophilic graph benchmarks.}
\label{tab:heterophily}
\begin{tabular}{lcccc}
\toprule
Dataset & Classifier-only & GraphLoRA & SRP & SRP vs.\ GraphLoRA \\
\midrule
Cornell
& $32.22\pm8.35$
& $45.56\pm14.34$
& $\mathbf{51.67\pm5.72}$
& $+6.11$ \\
Squirrel
& $23.08\pm1.98$
& $23.06\pm1.40$
& $\mathbf{25.12\pm1.47}$
& $+2.06$ \\
Chameleon
& $31.74\pm1.88$
& $31.43\pm1.44$
& $\mathbf{33.58\pm2.76}$
& $+2.15$ \\
\bottomrule
\end{tabular}
\end{table}

SRP outperforms GraphLoRA on all three datasets. The improvement is particularly clear on Squirrel and Chameleon, where GraphLoRA falls below the Classifier-only baseline while SRP remains above it. These results support the applicability of SRP beyond homophilic settings.

\subsection{Scalability and Runtime}
\label{sec:runtime_appendix}

The number of trainable SRP parameters is independent of the number of graph nodes and edges, while the dominant computation remains the frozen GNN backbone. Thus, SRP remains applicable to large graphs. In our experiments, including ogbn-arxiv with nearly $170{,}000$ nodes, all runs were conducted on a single NVIDIA RTX 3090 GPU with 24\,GB memory.

We additionally compare runtime against GPF, which has the lowest empirical per-epoch runtime among the prompt-learning baselines in our main experiments.

\begin{table}[ht]
\centering
\small
\caption{Per-epoch runtime in seconds. Lower is better.}
\label{tab:runtime}
\begin{tabular}{lcc}
\toprule
Dataset & GPF & SRP \\
\midrule
Cora     & 0.121 & \textbf{0.077} \\
CiteSeer & 0.131 & \textbf{0.084} \\
PubMed   & 0.678 & \textbf{0.508} \\
ENZYMES  & \textbf{0.221} & 0.292 \\
DD       & 0.191 & \textbf{0.136} \\
NCI1     & 0.342 & \textbf{0.297} \\
\bottomrule
\end{tabular}
\end{table}

SRP is faster than GPF on five of the six evaluated datasets, indicating that its spectral operations do not introduce substantial training-time overhead.

\subsection{Difference from GraphLoRA}
\label{sec:graphlora_difference}

GraphLoRA and SRP both introduce low-dimensional trainable components around a frozen GNN, but they address different adaptation objectives and construct different adaptation subspaces. GraphLoRA focuses on feature and structural distribution shifts between source and target graphs through a structure-aware low-rank adaptation branch. Its low-rank correction remains within the message-passing pathway:
\begin{equation}
    \Delta Z_{\mathrm{GraphLoRA}}^{(\ell)}
    =
    \widehat A\,\Delta W^{(\ell)\top}H^{(\ell)} .
\end{equation}

SRP instead targets the spectral capacity allocation of the frozen transformation. For
\begin{equation}
    W^{(\ell)}
    =U^{(\ell)}\Sigma^{(\ell)}V^{(\ell)\top},
\end{equation}
it explicitly constructs prompts from two complementary subspaces. The spectral channel can be written as
\begin{equation}
    P_{\mathrm{spec}}
    =
    \big(HV\odot\mu(\Sigma)\big)A_sQ,
    \qquad
    \mu_j
    =
    \operatorname{sigmoid}
    \!\left((\tau\sigma_{\max}-\sigma_j)\gamma\right),
\end{equation}
while the null-space channel is
\begin{equation}
    P_{\mathrm{null}}=HNA_nQ,
    \qquad
    \operatorname{col}(N)\subseteq\ker(W).
\end{equation}
The resulting layer is
\begin{equation}
    H^{(\ell+1)}
    =
    \phi\!\left(
    \widehat A W^{(\ell)\top}H^{(\ell)}
    +P_{\mathrm{spec}}+P_{\mathrm{null}}
    \right).
\end{equation}

Hence, SRP does not merely impose a low-rank parameter budget: it explicitly selects where adaptation is concentrated according to the singular geometry of the frozen transformation. GraphLoRA learns a flexible low-rank correction from graph-structural objectives, whereas SRP directs adaptation toward weak singular and null-space directions.

The position of the adaptation branch also differs. If
\begin{equation}
    H\Delta W^\top
    =\sum_i \phi_i c_i^\top ,
\end{equation}
then a GraphLoRA-style correction inside message passing becomes
\begin{equation}
    \widehat A\Delta W^\top H
    =
    \sum_i (1-\lambda_i)\phi_i c_i^\top ,
\end{equation}
and is therefore filtered by the propagation operator. In contrast,
\begin{equation}
    Z_{\mathrm{SRP}}
    =
    \widehat A W^\top H+p(H),
\end{equation}
injects $p(H)$ after aggregation, so its graph Fourier components are not multiplied again by $(1-\lambda_i)$ within the same layer. This gives SRP a complementary pathway for information suppressed by the frozen message-passing branch.

\subsection{Difference from Standard LoRA}
\label{sec:lora_difference}

Standard LoRA and SRP are both parameter-efficient adaptation methods, but differ in their adaptation target, trainable-subspace geometry, and interaction with graph propagation. LoRA freezes $W$ and learns an additive low-rank update
\begin{equation}
    W_{\mathrm{eff}}=W+\Delta W,
    \qquad
    \Delta W=BA,
    \qquad
    \operatorname{rank}(\Delta W)\leq r .
\end{equation}
For a message-passing GNN layer,
\begin{equation}
    Z_{\mathrm{LoRA}}
    =
    \widehat A(W+BA)^\top H
    =
    \widehat A W^\top H
    +\widehat A A^\top B^\top H .
\end{equation}
Thus, although the pretrained parameters remain frozen, LoRA changes the effective linear operator and parameterizes a low-rank subspace in weight-update space.

SRP instead leaves the original message-passing operator unchanged and learns a representation-space prompt:
\begin{equation}
    Z_{\mathrm{SRP}}
    =
    \widehat A W^\top H+p(H),
\end{equation}
where
\begin{equation}
    p(H)
    =
    \big[(HV\odot\mu(\Sigma))A_s+(HN)A_n\big]Q+b,
    \qquad
    \operatorname{col}(N)\subseteq\ker(W).
\end{equation}
The two channels therefore do not factorize a weight update; rather, they determine which information from the hidden representation is injected through the prompt.

Another distinction lies in the adaptation subspace. The low-rank constraint in LoRA controls the dimensionality of the update but does not explicitly determine how its energy is distributed across the singular directions of the frozen weight matrix. SRP instead derives its prompt-reading directions from the singular geometry of $W$ and redirects prompt capacity toward weak singular directions and $\ker(W)$.

This distinction is particularly relevant for GNNs. Let
\begin{equation}
    L=I-\widehat A
    =\Phi\Lambda\Phi^\top .
\end{equation}
If $A^\top B^\top H=\Phi C$, the LoRA correction becomes
\begin{equation}
    \Delta Z_{\mathrm{LoRA}}
    =
    \Phi(I-\Lambda)C,
\end{equation}
and is therefore filtered by the same graph propagation operator. By contrast, the SRP prompt is added after aggregation and is not multiplied again by $I-\Lambda$ in the current layer. In summary, LoRA provides generic low-rank operator adaptation, whereas SRP preserves the frozen operator and explicitly targets weak singular and null-space directions through a post-aggregation prompt.

\subsection{Clarification of the Singular-Value Groups in Figure~1}
\label{sec:figure1_clarification}

Let
\begin{equation}
    \sigma_1\geq\sigma_2\geq\cdots\geq\sigma_k>0
\end{equation}
denote the nonzero singular values of a frozen layer $W$, sorted in descending order. We define the five singular-value groups by
\begin{equation}
    b_q=\left\lfloor\frac{qk}{5}\right\rfloor,
    \qquad
    Q_q=\{b_{q-1}+1,\ldots,b_q\},
    \qquad q=1,\ldots,5,
\end{equation}
with $b_0=0$. Thus, $Q_1$ contains the directions with the largest singular values and $Q_5$ contains the weakest nonzero singular directions; group sizes differ by at most one. Exact zero-singular-value directions are treated separately by the null-space channel.

The quintile partition depends only on the spectrum of the frozen weight $W$. The downstream classifier $W_{\mathrm{down}}$ enters only through the downstream-demand term in DPMS and does not affect the group construction.

\subsection{Why SRP Shows Larger Gains on Node-Level Tasks}
\label{sec:node_graph_gap}

The difference between node- and graph-level performance can be attributed mainly to input dimensionality, null-space capacity, and graph-level readout.

\paragraph{Larger null-space capacity.}
For a first-layer transformation
\begin{equation}
    W\in\mathbb{R}^{h\times d_{\mathrm{in}}},
    \qquad
    W=U\Sigma V^\top,
\end{equation}
if $\operatorname{rank}(W)=k$, then
\begin{equation}
    \dim\ker(W)=d_{\mathrm{in}}-k
    \geq d_{\mathrm{in}}-h.
\end{equation}
Hence, when $d_{\mathrm{in}}\gg h$, the first frozen layer necessarily contains a large right null space. The node-classification datasets have input dimensions of $1433$ (Cora), $3703$ (CiteSeer), $500$ (PubMed), $500$ (Flickr), and $128$ (ogbn-arxiv), whereas the graph-classification datasets contain only $3$--$89$ input dimensions. With $h=128$, the null-space dimensions of the first-layer weights on Cora and CiteSeer are at least $1305$ and $3575$, respectively; on Cora we indeed observe a $1305$-dimensional null space. In contrast, when $d_{\mathrm{in}}\leq h$, the first layer may have full column rank and
\begin{equation}
    \dim\ker(W)=0.
\end{equation}
Graph-level datasets therefore provide substantially less capacity for the null-space channel, and SRP relies more heavily on spectral reversal or the projective fallback.

\paragraph{Graph-level readout.}
Node classification predicts directly from each node representation:
\begin{equation}
    \hat y_i=C(z_i+p_i),
\end{equation}
so a node-local prompt directly affects the prediction of node $i$. Graph classification instead aggregates node representations. With mean pooling,
\begin{equation}
    g=\frac{1}{N_g}\mathbf{1}^\top H,
    \qquad
    \Delta g=\frac{1}{N_g}\mathbf{1}^\top P.
\end{equation}
If $P=\Phi C$, then
\begin{equation}
    \Delta g
    =
    \frac{1}{N_g}
    \sum_j(\mathbf{1}^\top\phi_j)c_j^\top .
\end{equation}
For the combinatorial Laplacian of a connected graph,
\begin{equation}
    \mathbf{1}^\top\phi_j=0,
    \qquad j>1.
\end{equation}
Mean readout therefore preserves the constant component while cancelling many nonstationary node-local variations. Sum pooling has the same directional effect up to scaling, which can reduce the contribution of local information recovered by SRP.

\paragraph{Cross-graph consistency.}
For node classification on a single graph, define the graph frequency of the $j$-th feature direction by
\begin{equation}
    q_j
    =
    \frac{(Hv_j)^\top L(Hv_j)}
         {\|Hv_j\|_2^2}.
\end{equation}
This quantity has a consistent meaning on the same graph. For graph classification, however,
\begin{equation}
    q_{g,j}
    =
    \frac{(H_gv_j)^\top L_g(H_gv_j)}
         {\|H_gv_j\|_2^2},
\end{equation}
and the same $v_j$ can correspond to different graph frequencies across different graphs. Large variation in $q_{g,j}$ weakens the correspondence between a shared singular direction and complementary graph-frequency information. These factors help explain why the empirical gains are generally larger on node-level benchmarks.

\subsection{Why the Gains Are Larger on Some Smaller Datasets}
\label{sec:dataset_size}

The observed trend is better explained by spectral complement capacity than by dataset size alone, because input dimension and graph size are confounded in the current benchmarks. Cora and CiteSeer contain $2{,}708$ and $3{,}327$ nodes but have input dimensions of $1{,}433$ and $3{,}703$, respectively. In contrast, ogbn-arxiv contains $169{,}343$ nodes with only $128$ input dimensions, while Flickr contains $89{,}250$ nodes with $500$ input dimensions.

The exact right-null-space dimension is
\begin{equation}
    \dim\ker(W)
    =
    d_{\mathrm{in}}-\operatorname{rank}(W),
\end{equation}
which depends on the input dimension and weight rank rather than directly on the number of nodes. Cora and CiteSeer therefore provide substantially greater capacity for the null-space channel.

A second factor is local frequency heterogeneity. SRP currently uses a shared right-singular basis $V$ and a global spectral threshold,
\begin{equation}
    \mu_j
    =
    \operatorname{sigmoid}
    \!\left((\tau\sigma_{\max}-\sigma_j)\gamma\right).
\end{equation}
For a graph partitioned into regions $\mathcal{R}_1,\ldots,\mathcal{R}_M$, define
\begin{equation}
    x_{r,j}=H_{\mathcal{R}_r}v_j,
    \qquad
    q_{r,j}
    =
    \frac{x_{r,j}^\top L_{\mathcal{R}_r}x_{r,j}}
         {x_{r,j}^\top x_{r,j}}.
\end{equation}
When $\operatorname{Var}_r(q_{r,j})$ is large, the same singular direction can represent different graph-frequency characteristics across regions. A single global threshold can therefore become less precise as local frequency heterogeneity increases. Accordingly, the observed performance trend is more closely associated with input dimensionality, null-space capacity, and local frequency heterogeneity than with graph size itself.

\section{Computational Complexity}
\label{app:complexity}

\subsection{Compute Resources Usage.}
All experiments are conducted on a workstation running the Ubuntu 22.04 LTS operating system. The hardware configuration includes a 13th Gen Intel(R) Core(TM) i9-13900K CPU and an NVIDIA GeForce RTX 3090 GPU equipped with 24GB of memory. Furthermore, the software environment is built using CUDA 12.1 and Python 3.10.19, utilizing PyTorch 2.5.1 and PyTorch Geometric 2.7.0 for the model implementations.

\subsection{Complexity Analysis.}
SRP's per-layer forward pass: (i) $V$-projection $O(|\calV|kd_\text{in})$ (pre-computable); (ii) mask $O(|\calV|k)$; (iii) $A$-projection $O(|\calV|kr_s)$; (iv) $N$-projection $O(|\calV|md_\text{in})$ (pre-computable); (v) $B$-projection $O(|\calV|mr_s)$; (vi) $Q$-projection $O(|\calV|r_s d_\text{out})$.
Trainable-path cost: $O(|\calV|(k+m)r_s + |\calV|r_s d_\text{out})$, negligible vs.\ backbone's $O(|\calE|d_\text{in})$.

We analyze the time and space complexity of SRP under a two-layer GCN backbone
with $L = 2$ frozen layers.
Let $N$ denote the number of nodes, $|\mathcal{E}|$ the number of edges,
$d$ the raw input feature dimension, $h$ the hidden (and output) dimension,
$k = \operatorname{rank}(W) \le h$ the numerical rank of each frozen weight
$W \in \mathbb{R}^{h \times d_{\mathrm{in}}}$,
$m$ the null-space PCA dimension (\texttt{null\_pca\_dim}, default $m{=}64$),
and $r$ the shared bottleneck rank (\texttt{r\_shared}, default $r{=}8$).

\textbf{Initialisation.}
The initialisation of SRP involves two SVD computations and one null-space
projection, all performed once before training.
For each frozen weight $W_i$, the reduced SVD $W_i = U_i \Sigma_i V_i^\top$
costs $\mathcal{O}(k^2 d_{\mathrm{in},i})$ time
(thin SVD of an $h \times d_{\mathrm{in},i}$ matrix).
For the first layer, when $d > k + m$ holds and the dual-channel path is
activated, SRP additionally projects the node-feature matrix
$X \in \mathbb{R}^{N \times d}$ onto the null space via
$X_{\mathrm{null}} = X - X V_0 V_0^\top$,
costing $\mathcal{O}(N d k)$,
followed by a truncated SVD of the centred $X_{\mathrm{null}}$ to obtain the
top-$m$ null-space PCA basis $N_{\mathrm{mat}} \in \mathbb{R}^{d \times m}$,
which costs $\mathcal{O}(\min(N,d)^2 \max(N,d))$ but is performed on CPU once
and cached as a frozen buffer.
Summing across layers, the total initialisation time is
$\mathcal{O}(k^2 d + N d k + \min(N,d)^2 \max(N,d))$,
which for typical datasets satisfies $Nd \ll \min(N,d)^2\max(N,d)$,
so the dominant term is the one-time null-space PCA.
This cost is paid once before any gradient step and does not affect training
throughput.

\textbf{Per-forward-pass time complexity.}
During training and inference, \texttt{get\_prompt} is invoked once per GNN
layer per forward pass.
For the first layer operating in the dual-channel mode
($d > k + m$), the three dominant operations are
$h_v = h W_0^{-} V_0 \in \mathbb{R}^{N \times k}$ at cost $\mathcal{O}(Ndk)$,
$h_{\mathrm{null}} = h W_0^{-} N_{\mathrm{mat}} \in \mathbb{R}^{N \times m}$
at cost $\mathcal{O}(Ndm)$,
and the two bottleneck projections
$z_{\mathrm{spec}} = (h_v \odot \mathbf{s}) A \in \mathbb{R}^{N \times r}$
and $z_{\mathrm{null}} = h_{\mathrm{null}} B \in \mathbb{R}^{N \times r}$
at cost $\mathcal{O}(N(kr + mr))$,
followed by the output projection $z Q + \mathbf{b}$ at cost $\mathcal{O}(Nrh)$.
Since $k, m \ll d$ and $r \ll k, m$, the total for layer~0 is
$\mathcal{O}(Nd(k + m))$.
For the second layer operating under the projective fallback ($d_{\mathrm{in}}
= h \le k + m$), the cost is $\mathcal{O}(Nh^2 + Nhr)$,
which collapses to $\mathcal{O}(Nh^2)$.
The backbone GCN message-passing itself costs $\mathcal{O}(|\mathcal{E}|h +
Ndh)$ across the two layers.
Combining prompt and backbone, the overall per-forward-pass complexity is
\[
  \mathcal{O}\!\bigl(Nd(k + m) + Nh^2 + |\mathcal{E}|h\bigr)
  \;=\;
  \mathcal{O}\!\bigl(Ndh + |\mathcal{E}|h\bigr),
\]
where the equality uses $k \le h$ and $m \ll h$.
Thus SRP adds no asymptotic overhead beyond the backbone GNN's own
message-passing and projection costs.

\textbf{Parameter count and space complexity.}
SRP introduces only a small set of trainable parameters.
For layer~0 (dual-channel): $A_0 \in \mathbb{R}^{k \times r}$,
$B_0 \in \mathbb{R}^{m \times r}$, $Q_0 \in \mathbb{R}^{r \times h}$,
$\mathbf{b}_0 \in \mathbb{R}^{h}$, and the scalar threshold $\tau_0$,
contributing $kr + mr + rh + h + 1 = r(k + m + h) + h + 1$ parameters.
For layer~1 (projective fallback): $P_1 \in \mathbb{R}^{h \times r}$,
$Q_1 \in \mathbb{R}^{r \times h}$, $\mathbf{b}_1 \in \mathbb{R}^{h}$,
and $\tau_1$, contributing $2rh + h + 1$ parameters.
The total trainable parameter count is therefore
$r(k + m + 3h) + 2h + 2$.
The additional memory consists of frozen buffers:
$V_0 \in \mathbb{R}^{d \times k}$ and
$N_{\mathrm{mat}} \in \mathbb{R}^{d \times m}$ for layer~0,
and $V_1 \in \mathbb{R}^{h \times k}$ for layer~1,
totalling $d(k + m) + hk$ floating-point values.
The activation memory during a forward pass is
$\mathcal{O}(Nd)$ for the null-space projections of the full feature matrix,
which is at most $\mathcal{O}(Nd)$—identical to the GCN input tensor.
Overall, SRP's space overhead is $\mathcal{O}(d(k+m) + hk)$ in parameters
and buffers, with the number of \emph{trainable} parameters scaling as
$\mathcal{O}(r(k + m + h)) = \mathcal{O}(rh)$, far smaller than any
full fine-tuning or LoRA baseline of comparable expressiveness.

\end{document}